\documentclass[11pt]{article}

\usepackage[final]{acl}

\usepackage{microtype}

\usepackage{amsmath,amssymb,amsthm}
\usepackage{mathtools}
\usepackage{bm}

\usepackage{algorithm}
\usepackage{algorithmic}
\usepackage{booktabs}
\usepackage{multirow}
\usepackage{tabularx}
\usepackage{enumitem}

\usepackage{hyperref}

\newtheorem{theorem}{Theorem}
\newtheorem{lemma}{Lemma}
\newtheorem{proposition}{Proposition}
\newtheorem{assumption}{Assumption}
\newtheorem{remark}{Remark}

\newcommand{\inner}[2]{\left\langle #1, #2 \right\rangle}

\DeclareMathOperator{\rank}{rank}
\DeclareMathOperator{\EX}{\mathbb{E}}

\allowdisplaybreaks

\title{MONA: Muon Optimizer with Nesterov Acceleration for Scalable Language Model Training}

\author{
    Jiacheng Li,
    \ Jianchao Tan\thanks{Corresponding author.},
    \ Hongtao Xu,\\
    \ {\bf Jiaqi Zhang},
    \ {\bf Yifan Lu},
    \ {\bf Yerui Sun},
    \ {\bf Yuchen Xie},
    \ {\bf Xunliang Cai}, \\
    \ {\bf Meituan, Beijing, China} \\
    \texttt{\{lijiacheng14, tanjianchao02\}@meituan.com}
}

\begin{document}
\maketitle

\begin{abstract}
The Muon optimizer has recently offered a promising alternative to AdamW for large language model training, leveraging matrix orthogonalization to produce geometry-aware updates. However, like all first-order methods, Muon can become trapped in sharp local minima. In this work, we present \textbf{MONA}, an optimizer that bridges Muon's orthogonalization framework with curvature-aware acceleration. MONA adds an acceleration term directly into Muon’s gradient processing pipeline. This term is calculated from the exponential moving average of gradient differences.  We provide a detailed convergence analysis for MONA, showing that the acceleration term introduces curvature-sensitive corrections while preserving Muon's spectral-norm regularization. Empirically, MONA achieves better convergence and downstream task performance compared to both Muon and AdamW across three scales of Mixture-of-Experts pretraining, spanning from 1B to 68B parameters, with the largest model trained on 1 trillion tokens. Furthermore, we conduct supervised fine-tuning on the MOE-68B-A3B model and evaluate it on general capability, mathematical reasoning, and code generation benchmarks, where MONA achieves \textbf{SOTA} performance.
\end{abstract}

\section{Introduction}

\begin{figure}[t]
\centering
\includegraphics[width=\columnwidth]{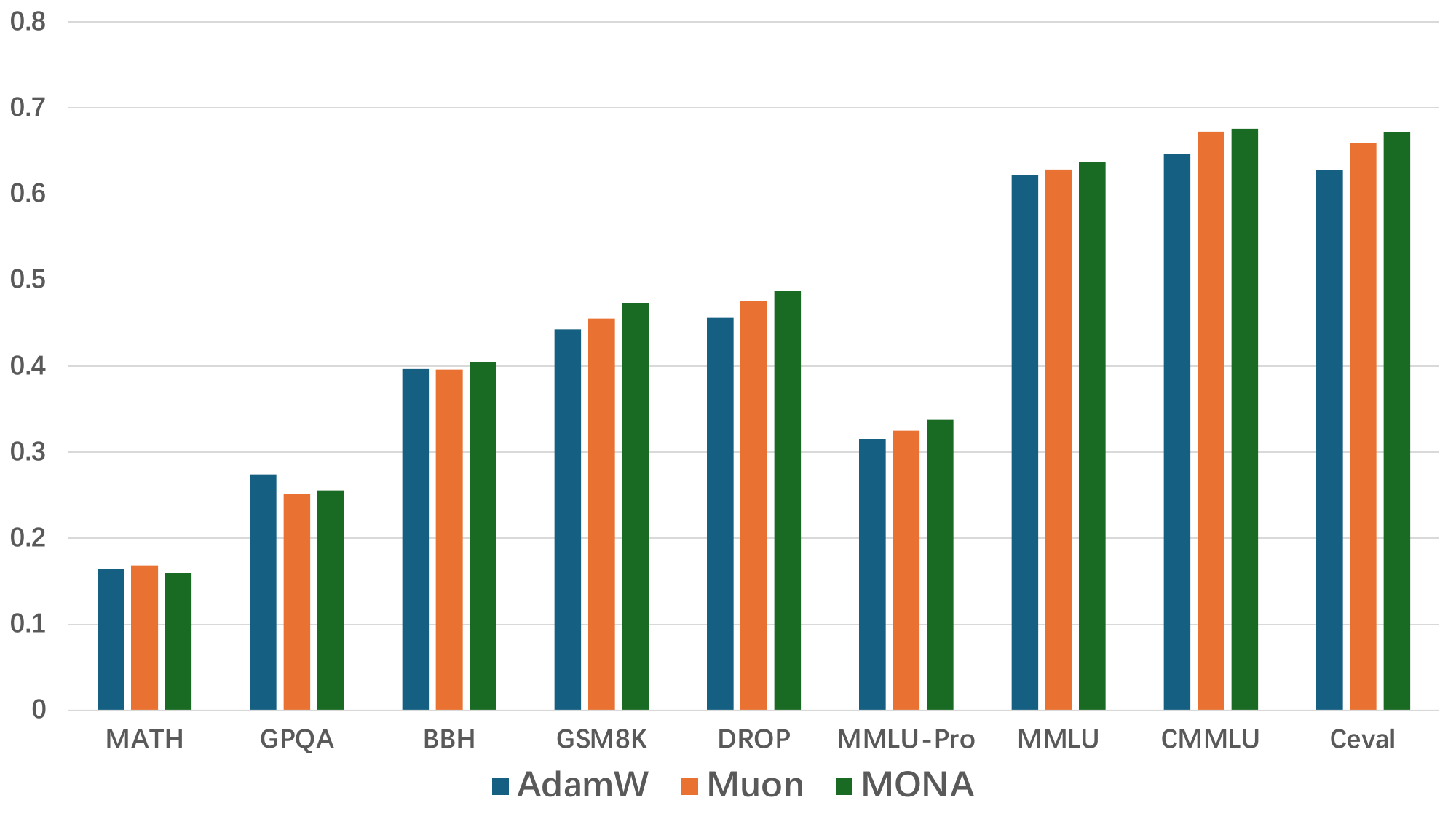}
\caption{General capability evaluation results for pretraining MOE-68B-A3B at 700B tokens. MONA consistently outperforms Muon and AdamW across multiple benchmarks.}
\label{fig:60b-3b-700B-bar}
\vskip -0.2in
\end{figure}

The choice of optimization algorithm\cite{robbins1951stochastic} is one of the most important decisions in training large language models (LLMs). For over a decade, Adam~\cite{kingma2014adam} and AdamW~\cite{loshchilov2017decoupled} have been the standard choice. However, as model sizes scale to the hundreds of billions of parameters~\cite{brown2020language, liu2024deepseek, team2025longcat, yang2025qwen3}, the requirement for optimizers with superior sample efficiency has increased.

Recently, Muon~\cite{jordan2024muon} has become a solid alternative. Instead of updating each parameter on its own, Muon views entire weight matrices as single geometric units. It then applies matrix orthogonalization to the momentum buffer. Muon is closely related to steepest descent under the spectral norm~\cite{li2025note}. ~\cite{liu2025muon} demonstrated that Muon achieves approximately $2\times$ computational efficiency compared to AdamW, training a 3B/16B-parameter MoE model on 5.7T tokens. Recent production deployments validate Muon's ability in large scale applications. Kimi K3~\cite{kimiteam2026kimik3openfrontier} and DeepSeek V4~\cite{deepseekai2026deepseekv4} both used Muon to train trillion parameter MoE models.

Despite these successes, Muon, like all first-order gradient methods, lacks explicit mechanisms for exploring the loss landscape beyond local gradient information. Once the optimizer enters the basin of attraction, Muon's updates are guided solely by the momentum-averaged gradient, with no principled way to distinguish between flat and sharp minima~\cite{keskar2016large}. This is particularly concerning for large-batch training~\cite{you2017large, you2019large}.

\cite{zhao2025alto} proposed ALTO, an optimizer adaptor that introduces an acceleration term based on the exponential moving average of gradient differences. The key insight is that $g_k - g_{k-1}$ implicitly encodes curvature information via $g_k - g_{k-1} \approx H_k(\theta_k - \theta_{k-1})$. By adding this acceleration term to the gradient, ALTO enables curvature-sensitive corrections and convergence to flatter solutions.

However, ALTO was designed as a general adaptor demonstrated on Adam and LAMB~\cite{you2019large}. Its integration with Muon presents challenges: Muon's orthogonalization is a highly nonlinear transformation, and it is unclear how curvature-aware acceleration interacts with this process. Moreover, ALTO's layer-wise learning rate regularization introduces complexity that may be unnecessary with Muon's inherently well-conditioned updates.

We present \textbf{MONA}, which seamlessly integrates the acceleration principle into Muon's orthogonalization framework. MONA applies the acceleration term before orthogonalization, transforming the raw gradient into a curvature-aware direction that is processed through Muon's standard pipeline.

Our contributions are threefold.
\begin{enumerate}[leftmargin=*]
\item \textbf{Algorithm Design.} MONA combines Muon's matrix orthogonalization with an acceleration term derived from gradient differences, preserving Muon's geometric structure while adding curvature awareness to the updates.

\item \textbf{Convergence Analysis.} We establish convergence rates for MONA in non-convex settings and describe the acceleration term's effect on curvature-sensitive corrections.

\item \textbf{Empirical Validation.} MONA outperforms both Muon and AdamW across three MoE model scales in pretraining, namely MOE-1B-A0d2B, MOE-6B-A0d5B, and MOE-68B-A3B, with the largest being trained on 1T tokens. We further conduct supervised fine-tuning on the MOE-68B-A3B model and evaluate it on general capability, mathematical reasoning, and code generation benchmarks, where MONA consistently achieves superior performance.
\end{enumerate}

\section{Related Work}

\subsection{Matrix-Aware Optimizers}

The standard approach handles parameters as independent scalars, updating them based on their own gradient history. Second-order methods such as K-FAC~\cite{martens2015optimizing} and Shampoo~\cite{gupta2018shampoo} use Kronecker-factored preconditioners, but their $O(N^{3/2})$ complexity limits scalability.

Muon~\cite{jordan2024muon} occupies a unique position: by applying polar decomposition to the momentum matrix, it achieves $O(N)$ complexity while respecting matrix geometry. The updates are closely related to steepest descent under the spectral norm~\cite{li2025note}.

\subsection{Muon Variants}

Scaling up Muon was explored by~\cite{liu2025muon}, who identified weight decay and per-parameter update scaling ($\gamma = 0.2\sqrt{\max(m, n)}$) as crucial for scaling Muon to billion-parameter models, along with distributed ZeRO-1-style optimization.
Adaptive variants were also investigated. AdaMuon~\cite{si2025adamuon} addresses Muon's lack of element-wise adaptivity by incorporating second-momentum estimation in the orthogonalized direction, achieving up to 40\% efficiency gains over AdamW.
ROOT~\cite{he2025root} proposes adaptive Newton iterations with size-specific coefficients for consistent precision across layers, and proximal optimization with soft thresholding to suppress outlier gradients in the momentum buffer before orthogonalization.

Efficiency improvements have been proposed as well. DropMuon~\cite{gruntkowska2025drop} introduces a randomized progressive training framework that updates only a subset of layers per step according to a randomized schedule, combining progressive training efficiency with layer-specific non-Euclidean updates. Dion~\cite{ahn2025dion} replaces Newton-Schulz iteration with amortized power iteration on a momentum buffer, avoiding full-matrix reconstruction and integrating cleanly with weight sharding. MuonBP~\cite{khaled2025muonbp} combines local block orthogonalization with global orthogonalization to eliminate the communication bottleneck under model parallelism while maintaining Muon's data efficiency.

\subsection{Acceleration in Optimization}

Adan~\cite{xie2024adan} develops a Nesterov momentum estimation method that estimates the gradient's first- and second-order moments for convergence acceleration in adaptive algorithms. Lion~\cite{chen2023symbolic} is a memory-efficient optimizer discovered through symbolic program search that uses only momentum tracking with sign-based updates at a uniform magnitude across all parameters.

ALTO~\cite{zhao2025alto} is the closest precursor. It introduces $a_k = \beta_a a_{k-1} + (1-\beta_a)(g_k - g_{k-1})$ and replaces $g_k$ with $g_k + \alpha a_k$. The insight is that $-\nabla\|\nabla f(\theta_k)\|^2 \approx H_k(\theta_k - \theta_{k-1}) \approx \bar{g}_k - \bar{g}_{k-1}$, pointing away from sharp minima.

\section{Preliminaries}

\subsection{Problem Setting}

Consider the stochastic non-convex problem
\begin{equation}
\min_{\theta \in \mathcal{D}} f(\theta), \quad f(\theta) = \EX_{\zeta}[\ell(\theta, \zeta)],
\end{equation}
where $\ell(\theta, \zeta)$ is the loss on sample $\zeta$ drawn from distribution $\mathcal{P}$. The parameter $\theta \in \mathbb{R}^d$ partitions into matrix-valued parameters $\{W^{(i)} \in \mathbb{R}^{m_i \times n_i}\}$ and vector-valued parameters. We focus on a single matrix $W \in \mathbb{R}^{m \times n}$.

At iteration $k$, we observe stochastic gradient $G_k = \nabla_W \ell(W_k, \zeta_k)$. The goal is to iteratively update $W_k$ toward a minimizer $W^*$.

\subsection{Muon Optimizer}

Muon~\cite{jordan2024muon} updates matrix parameters as:
\begin{align}
M_k &= \mu M_{k-1} + G_k, \label{eq:muon_momentum} \\
O_k &= \text{Newton-Schulz}(M_k), \label{eq:muon_ns} \\
W_{k+1} &= W_k - \eta \left(\gamma O_k + \lambda W_k\right). \label{eq:muon_update}
\end{align}

The Newton-Schulz iteration approximates the polar decomposition. Starting from $X_0 = M_k / \|M_k\|_F$,
\begin{multline}\label{eq:newton_schulz}
X_{t+1} = a X_t + b X_t X_t^\top X_t \\
+ c X_t X_t^\top X_t X_t^\top X_t,
\end{multline}
where $a = 3.4445$, $b = -4.7750$, $c = 2.0315$ ensure convergence for singular values in $[0, 1]$. After $T$ iterations (typically $T = 5$), $X_T \approx U V^\top$ where $M_k = U \Sigma V^\top$ is the SVD.

Following~\cite{liu2025muon}, the scaling factor matches Muon's RMS to AdamW:
\begin{equation}
\gamma = 0.2 \cdot \sqrt{\max(m, n)}.
\end{equation}

\subsection{Acceleration}

Acceleration computes:
\begin{align}
d_k &= g_k - g_{k-1}, \label{eq:alto_diff} \\
a_k &= \beta_a a_{k-1} + (1 - \beta_a) d_k, \label{eq:alto_accel} \\
\tilde{g}_k &= g_k + \alpha a_k, \label{eq:alto_grad}
\end{align}
where $\beta_a$ controls acceleration memory and $\alpha$ is the acceleration coefficient.

The theoretical motivation is that
\begin{multline}\label{eq:curvature_approx}
-\nabla \|\nabla f(\theta_k)\|^2 = -2 H_k \bar{g}_k 
\approx \bar{g}_k - \bar{g}_{k-1},
\end{multline}
where $H_k$ is the Hessian and $\bar{g}_k$ the full-batch gradient. The direction $-\nabla \|\nabla f(\theta_k)\|^2$ points away from sharp minima. ALTO uses gradient differences as a computationally efficient proxy.

\section{Methodology: MONA}





\subsection{Algorithm}

Pseudocode is in Algorithm~\ref{alg:mona}. For vector-valued parameters (embeddings, biases, etc.), MONA falls back to AdamW, following Muon's convention. Applying the acceleration term before momentum accumulation ensures that the momentum buffer captures curvature-aware directions for orthogonalization. Applying it after orthogonalization would destroy the orthogonal structure that Muon relies on.

\begin{algorithm}
\small
\caption{MONA Optimizer}
\label{alg:mona}
\begin{algorithmic}[1]
\REQUIRE $\eta$, $\mu$, $(\beta_a, \alpha)$, $\lambda$, NS steps $T$, scaling $\gamma$
\STATE $M_0 \leftarrow 0$, $A_0 \leftarrow 0$, $G_{0} \leftarrow 0$
\FOR{$k = 1, 2, 3, \ldots$}
\STATE $G_k \leftarrow \nabla_W \ell(W_k, \zeta_k)$
\STATE $D_k \leftarrow G_k - G_{k-1}$
\STATE $A_k \leftarrow \beta_a A_{k-1} + (1 - \beta_a) D_k$
\STATE $\tilde{G}_k \leftarrow G_k + \alpha A_k$
\STATE $M_k \leftarrow \mu M_{k-1} + \tilde{G}_k$
\STATE $O_k \leftarrow \text{Newton-Schulz}(M_k, T)$
\STATE $W_{k+1} \leftarrow W_k - \eta (\gamma O_k + \lambda W_k)$
\STATE $G_{k-1} \leftarrow G_k$
\ENDFOR
\end{algorithmic}
\end{algorithm}

\subsection{Geometric Intuition}

MONA's effectiveness stems from the interplay of two mechanisms.

\textbf{Spectral normalization (Muon).} Newton-Schulz ensures $O_k$ has singular values close to 1, preventing over-commitment to large-gradient directions. Muon performs steepest descent on the spectral-norm unit ball.

\textbf{Curvature-aware acceleration.} The term $A_k$ encodes how the gradient changes. Near sharp minima, $\|D_k\|$ is large, pushing toward flatter regions. In flat regions, $A_k$ is small, allowing stable convergence.

The combination works well together. Orthogonalization ensures geometrically well-conditioned updates, while acceleration enriches the input with curvature information for more informed direction selection.

\section{Theoretical Analysis}

We provide convergence analysis for MONA under standard assumptions. Detailed proofs are deferred to Appendix~\ref{sec:proofs}.

\subsection{Assumptions}

\begin{assumption}[L-smoothness]\label{assum:smooth}
The loss $\ell(W, \zeta)$ is $L$-smooth. For all $W, W'$ and $\zeta$,
\begin{multline}
\|\nabla_W \ell(W, \zeta) - \nabla_W \ell(W', \zeta)\|_F \\
\leq L \|W - W'\|_F.
\end{multline}

\end{assumption}

\begin{assumption}[Unbiased gradient with bounded variance]\label{assum:gradient}
The stochastic gradient satisfies $\EX[G_k \mid W_k] = \nabla f(W_k)$ and $\EX[\|G_k - \nabla f(W_k)\|_F^2] \leq \sigma^2$.
\end{assumption}

\begin{assumption}[Bounded gradient]\label{assum:bounded_grad}
There exists $G > 0$ such that $\|G_k\|_F \leq G$ a.s.
\end{assumption}

\begin{assumption}[Expected directional alignment]\label{assum:alignment}
There exists $\rho > 0$ such that $\EX[\inner{\nabla f(W_k)}{O_k} \mid W_k] \geq \rho \|\nabla f(W_k)\|_F^2$.
\end{assumption}

Assumption~\ref{assum:alignment} requires that the update direction $O_k$ has a positive expected correlation with the full gradient. This holds because Newton-Schulz preserves the column space of $M_k$, and the stochastic gradient (in expectation) lies within this space.

\subsection{Key Lemmas}

\begin{lemma}[Boundedness of acceleration]\label{lem:bounded_accel}
Under Assumptions~\ref{assum:gradient}--\ref{assum:bounded_grad}:
\begin{align}
\|A_k\|_F &\leq 2G, \\
\|\tilde{G}_k\|_F &\leq G(1 + 2|\alpha|).
\end{align}
\end{lemma}

\begin{lemma}[Momentum bound]\label{lem:momentum_bound}
Under the same assumptions,
\begin{equation}
\|M_k\|_F \leq \frac{G(1 + 2|\alpha|)}{1-\mu}.
\end{equation}
\end{lemma}

\begin{remark}[Newton-Schulz iteration]\label{rem:ns}
Newton-Schulz iteration converges to the polar decomposition $UV^\top$ exponentially fast in the number of iterations $T$. After $T$ steps, the approximation error satisfies $\|\tilde{O} - UV^\top\|_F \leq \epsilon_{\text{NS}}\sqrt{r}$ where $r = \text{rank}(M)$ and $\epsilon_{\text{NS}}$ decays exponentially with $T$. This is a standard result in numerical analysis, and with the standard coefficients and $T = 5$, the error is negligible for practical purposes.
\end{remark}

\subsection{Main Convergence Result}

\begin{theorem}[Non-convex convergence of MONA]\label{thm:nonconvex}
Let Assumptions~\ref{assum:smooth}--\ref{assum:alignment} hold. Run MONA with learning rate $\eta > 0$, momentum $\mu \in [0,1)$, acceleration $(\beta_a, \alpha)$ with $|\alpha| < 1/(1-\beta_a)$. Define:
\begin{align}
\bar{G} = G(1 + 2|\alpha|),& \quad C_1 = \tfrac{\bar{G}}{1-\mu}
\end{align}
\begin{align}
C_2 &= L\gamma^2 C_m / 2, \quad C_3 = \rho\gamma.
\end{align}

If $\eta \leq \min\{1/L, C_3/C_2\}$, then after $K$ iterations,
\begin{multline}
\frac{1}{K} \sum_{k=0}^{K-1} \EX\left[\|\nabla f(W_k)\|_F^2\right] \\
\leq \frac{f(W_0) - f^*}{\eta C_3 K} + \frac{\eta L C_4}{C_3},
\end{multline}

where $C_4 = \gamma^2 C_m / 2$ for $C_m = O(r)$ with $r = \rank(M_k)$.
\end{theorem}

\begin{proof}[Proof Sketch]
By L-smoothness,
\begin{multline}
f(W_{k+1}) \leq f(W_k) - \eta\gamma\inner{\nabla f(W_k)}{O_k} \\
+ \tfrac{\eta^2 L \gamma^2}{2}\|O_k\|_F^2.
\end{multline}
By Assumption~\ref{assum:alignment}, $\EX[\inner{\nabla f(W_k)}{O_k} \mid W_k] \geq \rho\|\nabla f(W_k)\|_F^2$. Since Newton-Schulz outputs approximately orthogonal matrices, $\|O_k\|_F^2 = O(r)$ where $r = \rank(M_k)$. Taking expectations and telescoping over $k = 0, \ldots, K-1$ yields the claimed first-moment bound. The condition on $\alpha$ ensures stability.
\end{proof}

\begin{remark}[Rate interpretation]
With $\eta = O(K^{-1/2})$, the bound becomes $O(K^{-1/2})$, matching SGD and AdamW. The acceleration affects constants, not the asymptotic rate.
\end{remark}

\subsection{Acceleration and Sharp Minimum Escape}

Near a stationary point $W^*$, with Hessian $H^*$,
\begin{multline}
f(W) \approx f(W^*) + \\
\tfrac{1}{2}\inner{W - W^*}{H^*(W - W^*)},
\end{multline}

and $\EX[G_k - G_{k-1}] \approx H^*(W_k - W_{k-1}) = -\eta\gamma H^* O_{k-1}$.

\begin{proposition}[Curvature-sensitive corrections]\label{prop:escape}
Near a minimum with high curvature ($\lambda_{\max}(H^*) \gg \lambda_{\min}(H^*) > 0$),
\begin{equation}
\EX[A_k] \approx -\eta\gamma H^* \sum_{j=0}^{k} (1-\beta_a)\beta_a^{k-j} O_{j-1}.
\end{equation}
For large eigenvalues (high-curvature directions), the acceleration is large, providing curvature-sensitive corrections. For small eigenvalues (flat directions), it is small, permitting convergence.
\end{proposition}

This formalizes that MONA selectively applies curvature-sensitive corrections. The selectivity arises from Hessian-dependent scaling—high-curvature directions amplify the correction; flat directions suppress it.

\subsection{Comparison with Baseline Muon}

In Muon, $O_k^{\text{Muon}} = \text{NS}(\mu M_{k-1} + G_k)$. In MONA, $O_k^{\text{MONA}} = \text{NS}(\mu M_{k-1} + G_k + \alpha A_k)$. The difference $\alpha A_k$ is approximately proportional to the negative Hessian-weighted average of past directions, biasing orthogonalization toward lower-curvature directions.

\section{Experiments}

\begin{table*}[t]
\centering
\small
\caption{General capability evaluation results for MOE-68B-A3B at 700B tokens. Scores are reported as mean $\pm$ std.}
\label{tab:general-capability}
\begin{tabular}{lccc}
\toprule
\textbf{Benchmark} & \textbf{AdamW} & \textbf{Muon} & \textbf{MONA} \\
\midrule
MMLU-FewShot~\cite{hendrycks2020measuring} & 0.6218$\pm$0.0098 & 0.6281$\pm$0.0098 & \textbf{0.6373}$\pm$0.0097 \\
MMLU-Pro-FewShot~\cite{wang2024mmlu} & 0.3150$\pm$0.0169 & 0.3250$\pm$0.0170 & \textbf{0.3375}$\pm$0.0172 \\
CMMLU-FewShot~\cite{li2306cmmlu} & 0.6464$\pm$0.0090 & 0.6723$\pm$0.0088 & \textbf{0.6756}$\pm$0.0088 \\
CEVAL-FewShot~\cite{huang2023c} & 0.6274$\pm$0.0263 & 0.6586$\pm$0.0255 & \textbf{0.6717}$\pm$0.0251 \\
BBH-FewShot~\cite{suzgun2023challenging} & 0.3966$\pm$0.0125 & 0.3962$\pm$0.0126 & \textbf{0.4049}$\pm$0.0126 \\
Math (CoT)~\cite{hendrycks2021measuring} & 0.1648$\pm$0.0103 & \textbf{0.1684}$\pm$0.0104 & 0.1594$\pm$0.0101 \\
GPQA-FewShot~\cite{rein2023gpqa} & \textbf{0.2743}$\pm$0.0365 & 0.2518$\pm$0.0355 & 0.2555$\pm$0.0356 \\
DROP (CoT)~\cite{dua2019drop} & 0.4560$\pm$0.0100 & 0.4757$\pm$0.0100 & \textbf{0.4868}$\pm$0.0100 \\
GSM8K (CoT)~\cite{cobbe2021training} & 0.4426$\pm$0.0270 & 0.4550$\pm$0.0271 & \textbf{0.4734}$\pm$0.0272 \\
\midrule
\textbf{Average} & 0.4382 & 0.4478 & \textbf{0.4557} \\
\bottomrule
\end{tabular}
\end{table*}

\begin{table*}[t]
\centering
\small
\caption{Code generation and mathematical reasoning evaluation results for
MOE-68B-A3B at 700B tokens. Scores are reported as mean $\pm$ std.
Knowledge-Specific is the average of
hellaswag~\cite{zellers2019hellaswag},
commonsenseqa~\cite{talmor-etal-2019-commonsenseqa},
openbookqa~\cite{OpenBookQA2018},
piqa~\cite{Bisk2020},
siqa~\cite{sap2019socialIQa}, and
winogrande~\cite{sakaguchi2019winogrande}.}
\label{tab:code-math}

\setlength{\tabcolsep}{4pt}
\begin{tabularx}{\textwidth}{
    >{\raggedright\arraybackslash}p{0.4\textwidth}
    *{3}{>{\centering\arraybackslash}X}
}
\toprule
\textbf{Benchmark}
& \textbf{AdamW}
& \textbf{Muon}
& \textbf{MONA} \\
\midrule
HumanEval+~\cite{liu2024evaluating}
& \textbf{0.2927}$\pm$0.0699
& 0.2866$\pm$0.0694
& 0.2805$\pm$0.0690 \\

Multiple~\cite{cassano2208multipl}
& 0.2519$\pm$0.0326
& 0.2239$\pm$0.0313
& \textbf{0.2848}$\pm$0.0343 \\

MBPP+~\cite{liu2024evaluating} 3-shot
& \textbf{0.5265}$\pm$0.0504
& 0.5053$\pm$0.0505
& 0.5265$\pm$0.0504 \\

BigCodeBench~\cite{allal2022framework}
& 0.3123$\pm$0.0269
& 0.3061$\pm$0.0268
& \textbf{0.3377}$\pm$0.0275 \\

LiveCodeBench~\cite{jain2025livecodebench}
& 0.0358$\pm$0.0219
& 0.0430$\pm$0.0238
& \textbf{0.0502}$\pm$0.0257 \\

LeetCode~\cite{xia2025leetcodedataset}
& 0.0540$\pm$0.0209
& \textbf{0.0847}$\pm$0.0249
& 0.0709$\pm$0.0232 \\

Bigmath Mathematics~\cite{albalak2025big}
& 0.4627$\pm$0.0173
& 0.4493$\pm$0.0182
& \textbf{0.4733}$\pm$0.0173 \\

ProofWriter~\cite{tafjord2021proofwriter}
& 0.4995$\pm$0.0143
& \textbf{0.5312}$\pm$0.0225
& 0.5169$\pm$0.0142 \\

Knowledge-Specific
& 0.6575$\pm$0.0142
& 0.7021$\pm$0.0228
& \textbf{0.7025}$\pm$0.0139 \\

CRUXEval~\cite{gu2024cruxeval}
& \textbf{0.3563}$\pm$0.0235
& 0.3362$\pm$0.0232
& 0.3500$\pm$0.0234 \\

FullStackBench~\cite{cheng2024fullstack}
& 0.2567$\pm$0.0145
& \textbf{0.3044}$\pm$0.0153
& 0.2975$\pm$0.0153 \\

HumanEval+ 3-shot
& \textbf{0.3110}$\pm$0.0711
& 0.3049$\pm$0.0707
& 0.3049$\pm$0.0707 \\
\midrule
\textbf{Average}
& 0.3353
& 0.3397
& \textbf{0.3495} \\
\bottomrule
\end{tabularx}
\end{table*}

\begin{figure}[h!]
\centering
\includegraphics[width=0.9\linewidth]{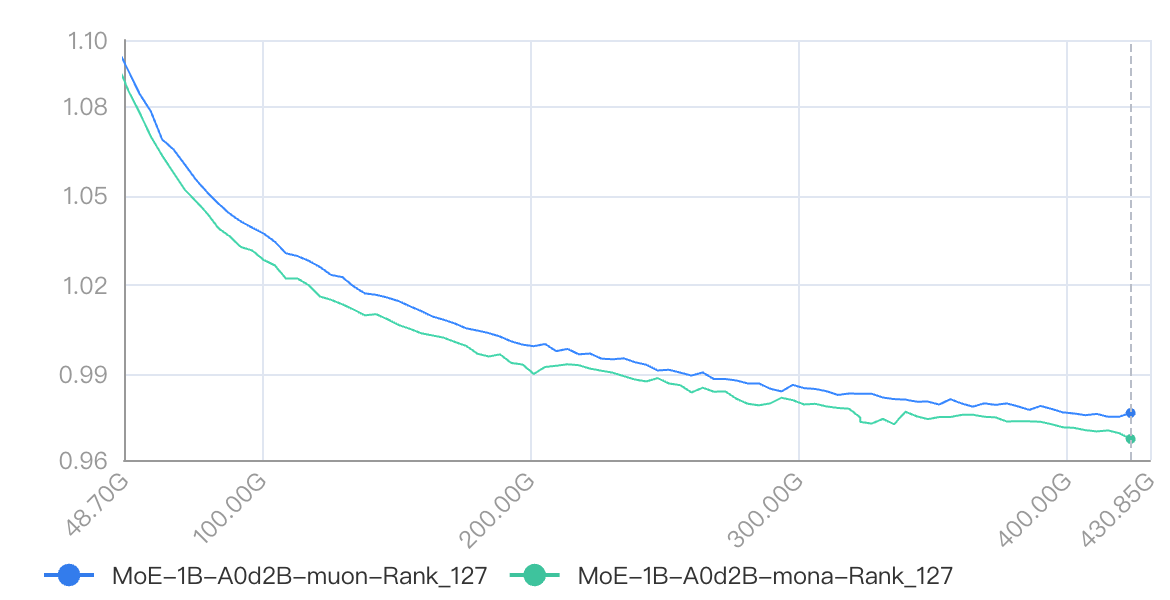}
\caption{Validation loss on Code-Valid for MOE-1B-A0d2B.}
\label{fig:200m-code}
\end{figure}

\begin{figure}[h!]
\centering
\includegraphics[width=0.9\linewidth]{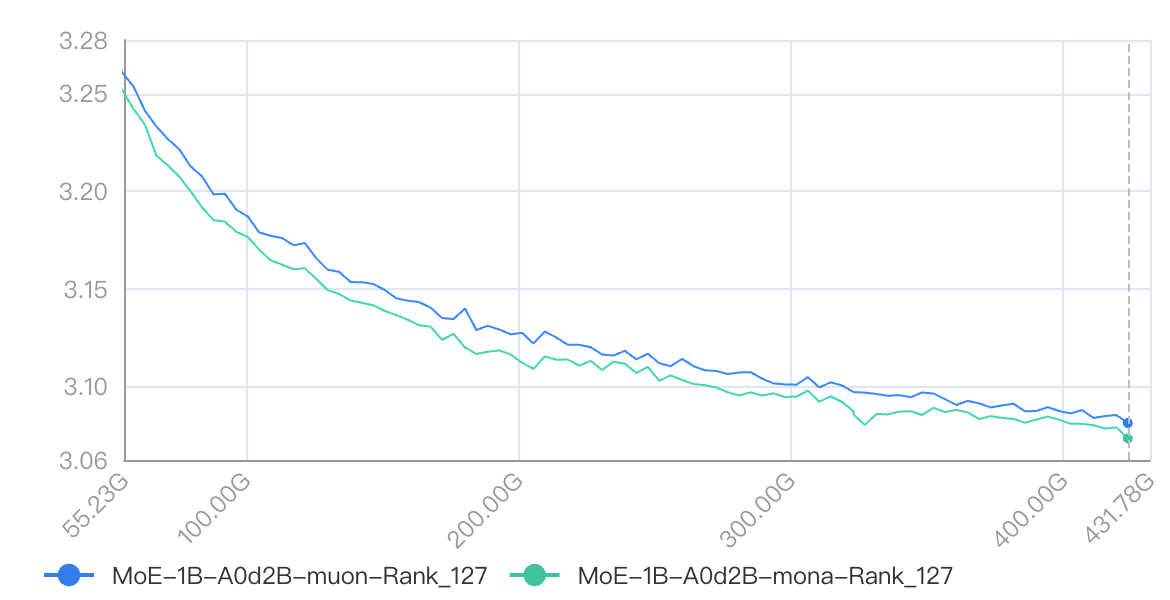}
\caption{Validation loss on General-English-Text for MOE-1B-A0d2B.}
\label{fig:200m-en_book}
\end{figure}

\begin{figure}[h!]
\centering
\includegraphics[width=0.9\linewidth]{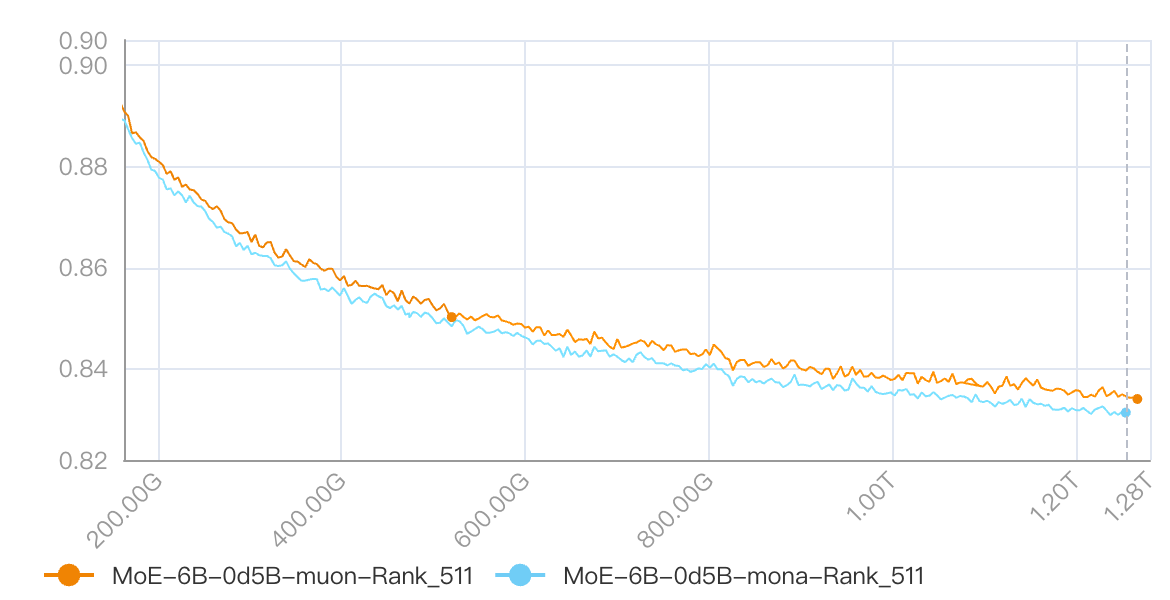}
\caption{Validation loss on Code-Valid for MOE-6B-A0d5B.}
\label{fig:500m-code}
\end{figure}

\begin{figure}[h!]
\centering
\includegraphics[width=0.9\linewidth]{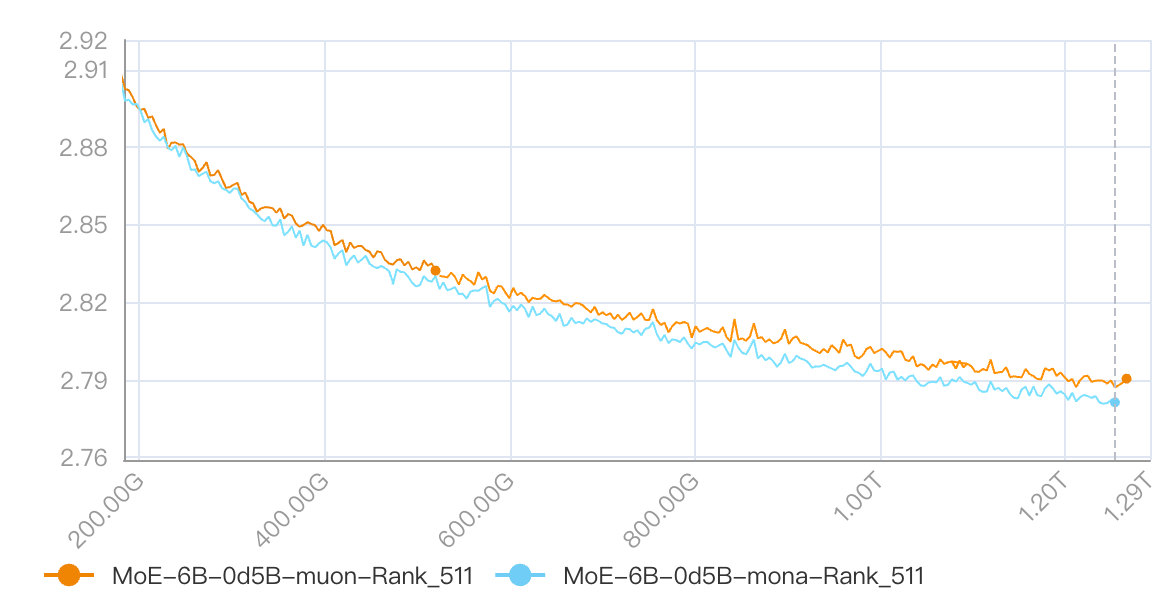}
\caption{Validation loss on General-English-Text for MOE-6B-A0d5B.}
\label{fig:500m-en_book}
\end{figure}

\subsection{Pretraining}

We pretrain three MoE~\cite{jacobs1991adaptive} language models of increasing scale based on the LongCat architecture~\cite{team2025longcat}, specifically ScMoE~\cite{cai2024shortcut} with MLA~\cite{liu2024deepseekmla, vaswani2017attention}, comparing MONA against Muon. All models employ MLA and train on sequences of length 8192. We monitor validation loss on four held-out domains covering code, mathematical reasoning, general English text, and Chinese academic text to assess convergence behavior across diverse capabilities. For all the experimental curves shown in the figures, except for the optimizer selection, we kept all hyperparameters, including learning rate, learning rate scheduling, and batch size, consistent throughout the experiments.

\textbf{MOE-1B-A0d2B.} The smallest model uses 10 transformer layers with 768 hidden dimensions, 16 attention heads, 128 experts with 256 FFN hidden size each, and top-8 routing. We train for approximately 400B tokens. As shown in Figure~\ref{fig:200m-code} and Figure~\ref{fig:200m-en_book}, MONA consistently achieves lower validation loss than Muon across both validation domains, with additional results on mathematical reasoning and Chinese academic text provided in Appendix~\ref{sec:additional-curves}. 

\begin{table*}[h]
\centering
\small
\caption{BigCode evaluation results for MOE-68B-A3B. Scores are reported as mean $\pm$ std. ckpt\_$-1$ is the checkpoint before SFT.}
\label{tab:bigcode}
\begin{tabular}{lccc}
\toprule
\textbf{Task} & \textbf{ckpt\_$-1$} & \textbf{Muon} & \textbf{MONA} \\
\midrule
\textsc{Multiple} & 0.6082$\pm$0.0366 & 0.7836$\pm$0.0302 & \textbf{0.7851}$\pm$0.0298 \\
\textsc{MBPP+} & 0.5423$\pm$0.0503 & 0.7778$\pm$0.0420 & \textbf{0.7857}$\pm$0.0414 \\
\textsc{HumanEval+} & 0.6829$\pm$0.0714 & 0.8720$\pm$0.0513 & \textbf{0.8841}$\pm$0.0491 \\
\textsc{TorchDataEval}~\cite{zan-etal-2022-language} & 0.7740$\pm$0.0328 & \textbf{0.8109}$\pm$0.0308 & 0.8077$\pm$0.0309 \\
\textsc{DS1000}~\cite{lai2023ds} & 0.2860$\pm$0.0332 & 0.3395$\pm$0.0350 & \textbf{0.3703}$\pm$0.0359 \\
\textsc{BigCodeBench} & 0.0658$\pm$0.0144 & 0.5526$\pm$0.0289 & \textbf{0.5553}$\pm$0.0289 \\
\textsc{HumanEval+} 3-shot & 0.0061$\pm$0.0120 & 0.8598$\pm$0.0533 & \textbf{0.8659}$\pm$0.0523 \\
\bottomrule
\end{tabular}
\end{table*}

\textbf{MOE-6B-A0d5B.} The medium-scale model also uses 10 transformer layers but increases to 1536 hidden dimensions, 128 experts of 1024 FFN hidden size each, and top-6 routing, training for approximately 1.2T tokens. Figure~\ref{fig:500m-code} and Figure~\ref{fig:500m-en_book} show that MONA maintains its advantage at this scale, again outperforming Muon across both validation domains (additional results are provided in Appendix~\ref{sec:additional-curves}). The more volatile curves compared to the MoE-1B-A0d2B model indicate a more complex optimization landscape.

\textbf{MOE-68B-A3B.} The largest model scales to 14 transformer layers with 3072 hidden dimensions, 32 attention heads, 256 experts with 1024 FFN hidden size each, and top-12 routing~\cite{liu2026scalingembeddingsoutperformsscaling}. We train for approximately 700B tokens and evaluate the intermediate checkpoint on a comprehensive suite of benchmarks covering general capability, mathematical reasoning, and code generation, comparing MONA against both Muon and AdamW.

Table~\ref{tab:general-capability} reports the results on general capability benchmarks. MONA achieves the highest average score (0.4557) across all three optimizers, outperforming Muon (0.4478) and AdamW (0.4382). Notably, MONA shows consistent improvements on MMLU-FewShot, MMLU-Pro, CMMLU-FewShot, CEVAL-FewShot, BBH-FewShot, DROP, and GSM8K, demonstrating that curvature-aware acceleration helps the model develop stronger general reasoning and mathematical capabilities.

Table~\ref{tab:code-math} reports the results on code generation and specialized mathematics benchmarks. MONA again achieves the highest average, compared to Muon and AdamW. MONA delivers particularly strong gains on \textsc{Multiple}, \textsc{BigCodeBench}, and \textsc{LiveCodeBench}, indicating that the acceleration term's exploration of flatter minima enables the model to learn more transferable code representations.

\subsection{Supervised Fine-Tuning and Evaluation}

\begin{figure}[h!]
\centering
\includegraphics[width=\linewidth]{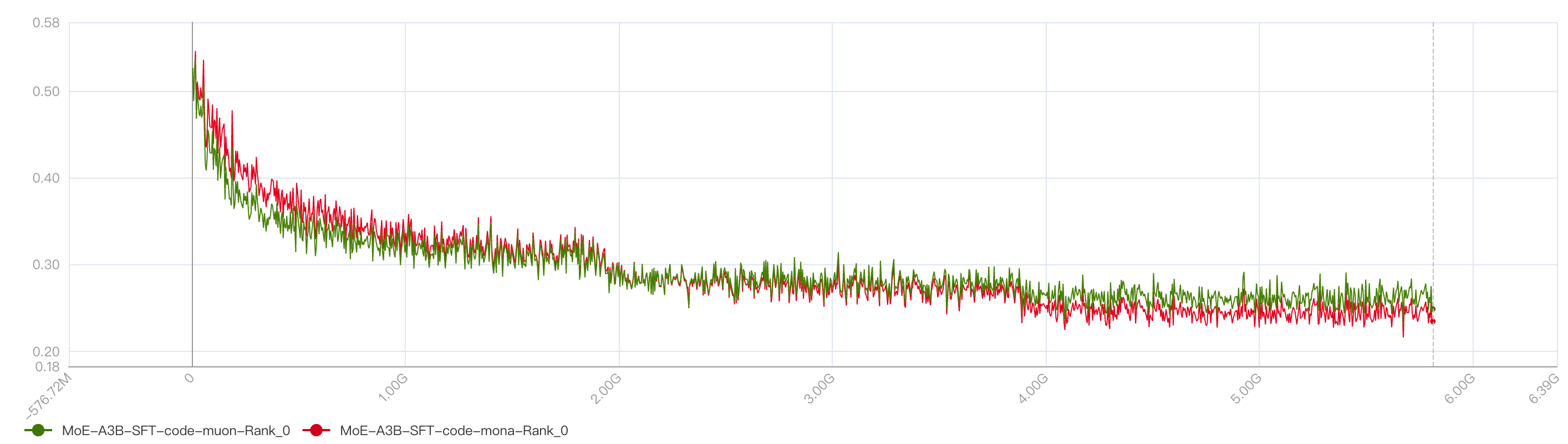}
\caption{SFT training loss on code data for the MOE-68B-A3B model. MONA-pretrained checkpoint (red) achieves lower loss than Muon (green) throughout training, with a larger gap emerging in later epochs.}
\label{fig:3B-SFT}
\end{figure}

To assess the practical utility of MONA-optimized models beyond pretraining, we conduct supervised fine-tuning (SFT) on the MOE-68B-A3B model using high-quality code data. The SFT stage employs Adam with a peak learning rate of cosine decay, training for 3 epochs on approximately 2B tokens with a maximum sequence length of 32k. We compare the MONA-pretrained and Muon-pretrained checkpoints under identical SFT settings.

Figure~\ref{fig:3B-SFT} shows the SFT training loss curves. The MONA-pretrained model consistently achieves lower training loss than the Muon-pretrained baseline throughout all three epochs, with the gap widening in the later stages of training. This suggests that the curvature-aware acceleration in MONA not only improves pretraining convergence but also produces initializations that are better suited for downstream adaptation.

Table~\ref{tab:bigcode} reports the results on the BigCode evaluation suite~\cite{zhuo2025bigcodebench}. The ckpt\_$-1$ column shows that the pretrained base model performs poorly on most code generation tasks, confirming that the subsequent SFT stage is essential for unlocking programming capabilities. After SFT, both Muon and MONA pretrained models achieve large improvements across all benchmarks. Comparing the two, MONA achieves higher scores than Muon on 6 out of 7 tasks. Notably, MONA delivers substantial gains on \textsc{MBPP+}, \textsc{HumanEval+}, and \textsc{DS1000}, etc., indicating that the curvature-aware acceleration during pretraining enables the model to learn more robust code representations that transfer better to downstream programming tasks.

\subsection{Efficient Deployment with Reduced Overhead}
\label{sec:memory-overhead}
MONA introduces two additional state buffers compared to Muon: the previous gradient $G_{k-1}$ and the acceleration buffer $A_k$, both stored in full precision during training. While this overhead is manageable on large-scale training clusters, it can be prohibitive for researchers and practitioners with limited GPU memory. To address this, we explore two complementary strategies for reducing the memory footprint of the acceleration term.

The first strategy is low-precision quantization. We implement \textbf{MONA-Lite}, a variant that stores $G_{k-1}$ and $A_k$ in bfloat16 (BF16) rather than float32 (FP32). This reduces the memory overhead of these auxiliary states by approximately 50\% without requiring any changes to the update equations or the orthogonalization pipeline. 

The second strategy is streaming gradient computation, an in-place engineering optimization that eliminates the $G_{k-1}$ buffer entirely. The original update maintains both $G_{k-1}$ and $A_{k-1}$ (overhead: two gradient-sized buffers):
\begin{align}
D_k &= G_k - G_{k-1}, \\
A_k &= \beta_a A_{k-1} + (1-\beta_a) D_k, \\
\tilde{G}_k &= G_k + \alpha A_k.
\end{align}

The streaming gradient computation algorithm reformulates the acceleration update using a state buffer $S_k$ that overwrites the acceleration slot in-place (overhead: one gradient-sized buffer):
\begin{align}
A_k &= S_{k-1} + (1-\beta_a) G_k, \label{eq:stream-recover}\\
G_k &\leftarrow G_k + \alpha A_k, \label{eq:stream-inject}\\
S_k &= \beta_a A_k - (1-\beta_a) G_k. \label{eq:stream-store}
\end{align}

At each step, Equation~\eqref{eq:stream-recover} recovers $A_k$ from the in-place state $S_{k-1}$; Equation~\eqref{eq:stream-inject} injects the acceleration directly into the gradient buffer; and Equation~\eqref{eq:stream-store} computes the next state $S_k$ and overwrites the acceleration slot. Algorithm~\ref{alg:mona-streaming} integrates this into MONA's full pipeline. The gradient computation algorithm is equivalent to the original algorithm, but saves one gradient-sized buffer in GPU memory.

\begin{algorithm}
\small
\caption{MONA with Streaming Gradient Computation}
\label{alg:mona-streaming}
\begin{algorithmic}[1]
\REQUIRE $\eta$, $\mu$, $(\beta_a, \alpha)$, $\lambda$, NS steps $T$, scaling $\gamma$
\STATE $M_0 \leftarrow 0$, $S_0 \leftarrow 0$
\FOR{$k = 1, 2, 3, \ldots$}
\STATE $G_k \leftarrow \nabla_W \ell(W_k, \zeta_k)$
\STATE $A_k \leftarrow S_{k-1} + (1 - \beta_a) G_k$ \COMMENT{recover $A_k$ in-place}
\STATE $G_k \leftarrow G_k + \alpha A_k$ \COMMENT{inject acceleration}
\STATE $S_k \leftarrow \beta_a A_k - (1 - \beta_a) G_k$ \COMMENT{overwrite slot with $S_k$}
\STATE $M_k \leftarrow \mu M_{k-1} + G_k$
\STATE $O_k \leftarrow \text{Newton-Schulz}(M_k, T)$
\STATE $W_{k+1} \leftarrow W_k - \eta (\gamma O_k + \lambda W_k)$
\ENDFOR
\end{algorithmic}
\end{algorithm}

\begin{figure}[h!]
\centering
\includegraphics[width=0.8\linewidth]{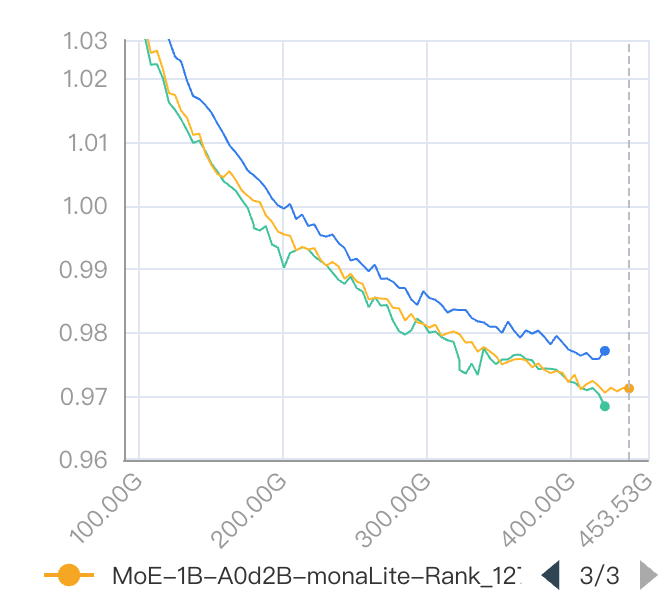}
\caption{Validation loss on Code\_Valid for the MOE-1B-A0d2B model, extending Figure~\ref{fig:200m-code} with the MONA-Lite curve (yellow). MONA-Lite closely tracks FP32 MONA while both maintain a clear advantage over Muon.}
\label{fig:mona-lite-code}
\end{figure}

\begin{figure}[h!]
\centering
\includegraphics[width=0.8\linewidth]{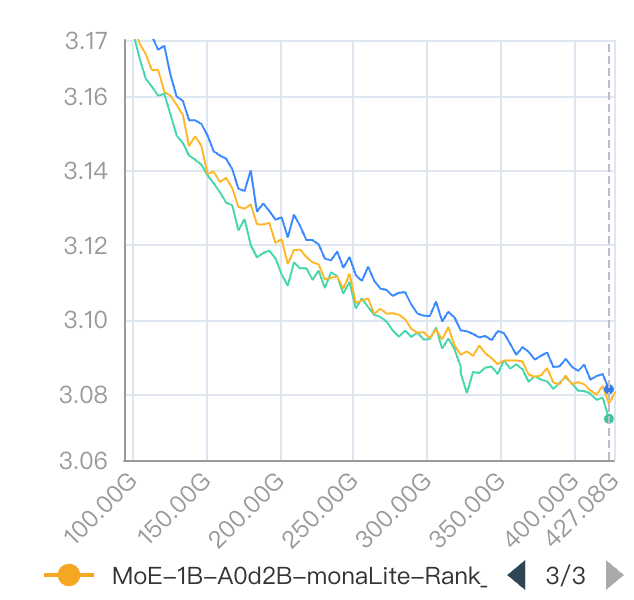}
\caption{Validation loss on General-English-Text for the MOE-1B-A0d2B model, extending Figure~\ref{fig:200m-en_book} with the MONA-Lite curve (yellow). As in the code domain, MONA-Lite remains indistinguishable from FP32 MONA while outperforming Muon.}
\label{fig:mona-lite-en}
\end{figure}

When combined, BF16 quantization and streaming computation reduce the total extra memory overhead by about 75\% compared to standard FP32 MONA. Streaming cuts it by 50\%, and quantization cuts it by another 50\%. This makes the accelerated optimizer practical for resource-constrained settings without sacrificing its curvature-aware benefits.

We evaluate both strategies by pretraining the MOE-1B-A0d2B model under identical hyperparameters. Figure~\ref{fig:mona-lite-code} builds on the code validation plot in Figure~\ref{fig:200m-code}, adding the MONA-Lite curve (yellow) alongside Muon (blue) and FP32 MONA (green). Similarly, Figure~\ref{fig:mona-lite-en} extends Figure~\ref{fig:200m-en_book} to the general English text domain. In both cases, MONA-Lite closely tracks the FP32 MONA curve throughout training, while both optimizers maintain a clear advantage over Muon. This demonstrates that the acceleration term can be safely compressed to BF16 and that streaming computation does not affect training quality across diverse evaluation domains.

We also measured the training speed overhead. Appendix~\ref{sec:time-overhead} shows MONA is about 1\% slower inside the optimizer step, but this difference disappears at the iteration level, so overall training time is essentially the same as Muon.

\section{Conclusion}

We propose MONA, an improved variant of the Muon optimizer, which integrates curvature-aware acceleration into its matrix orthogonalization framework. By augmenting gradients with an exponential moving average of gradient differences before orthogonalization, MONA endows Muon with curvature-sensitive corrections while preserving all of Muon's geometric benefits. Our theoretical analysis proves convergence under standard assumptions and shows how the acceleration term introduces curvature-sensitive corrections.

Empirically, MONA achieves lower validation loss than both Muon and AdamW across three scales of MoE models (1B to 68B parameters) on code, mathematical reasoning, and general text. At the 68B scale, MONA delivers superior general capability and code generation scores, and its pretrained models achieve higher BigCode evaluation results after code-specific SFT. To reduce the overhead of the acceleration term, we further introduce MONA-Lite, which combines BF16 quantization with streaming gradient computation to cut the extra memory by approximately 75\% without sacrificing training quality.

\section{Limitations}

\textbf{Hyperparameter tuning.} MONA introduces two additional hyperparameters $(\beta_a, \alpha)$ compared to Muon, which adds a tuning cost. However, we observe a consistent relationship between them across all experiments: $\alpha = -1 / (2(1-\beta_a))$. In practice, we set $\beta_a = 0.99$ for the MOE-1B-A0d2B model, $\beta_a = 0.98$ for the MOE-6B-A0d5B model, and $\beta_a = 0.975$ for the MOE-68B-A3B model. This shared rule largely alleviates the tuning overhead.

\textbf{Memory overhead.} Even with BF16 quantization and streaming gradient computation (MONA-Lite), the acceleration buffer $A_k$ still introduces approximately half a gradient's worth of extra memory. Users should balance this residual overhead against the observed training gains when deploying MONA in memory-constrained environments.

\section*{Acknowledgments}
We thank Xinyue Zhao and Xiang Li at Meituan for their support in implementing optimizer operators across multiple hardware backends in our training framework. We also thank Pingwei Sun, Jiayu Qin, Zixu Jiang, Feiye Huo, and Yuxuan Hu at Meituan for their helpful discussions.

\nocite{sutskever2013importance}
\nocite{micikevicius2017mixed}
\nocite{li2026wisca}
\bibliography{mona_references}

\begin{thebibliography}{60}
\providecommand{\natexlab}[1]{#1}

\bibitem[{Ahn et~al.(2025)Ahn, Xu, Abreu, Fan, Magakyan, Sharma, Zhan, and Langford}]{ahn2025dion}
Kwangjun Ahn, Byron Xu, Natalie Abreu, Ying Fan, Gagik Magakyan, Pratyusha Sharma, Zheng Zhan, and John Langford. 2025.
\newblock Dion: Distributed orthonormalized updates.
\newblock \emph{arXiv preprint arXiv:2504.05295}.

\bibitem[{Albalak et~al.(2025)Albalak, Phung, Lile, Rafailov, Gandhi, Castricato, Singh, Blagden, Xiang, Mahan, and Haber}]{albalak2025big}
Alon Albalak, Duy Phung, Nathan Lile, Rafael Rafailov, Kanishk Gandhi, Louis Castricato, Anikait Singh, Chase Blagden, Violet Xiang, Dakota Mahan, and Nick Haber. 2025.
\newblock \href {https://arxiv.org/abs/2502.17387} {Big-math: A large-scale, high-quality math dataset for reinforcement learning in language models}.
\newblock \emph{Preprint}, arXiv:2502.17387.

\bibitem[{Allal et~al.(2022)Allal, Muennighoff, Umapathi, Lipkin, and Von~Werra}]{allal2022framework}
Loubna~Ben Allal, Niklas Muennighoff, Logesh~Kumar Umapathi, Ben Lipkin, and Leandro Von~Werra. 2022.
\newblock A framework for the evaluation of code generation models.

\bibitem[{Bisk et~al.(2020)Bisk, Zellers, Bras, Gao, and Choi}]{Bisk2020}
Yonatan Bisk, Rowan Zellers, Ronan~Le Bras, Jianfeng Gao, and Yejin Choi. 2020.
\newblock Piqa: Reasoning about physical commonsense in natural language.
\newblock In \emph{Thirty-Fourth AAAI Conference on Artificial Intelligence}.

\bibitem[{Brown et~al.(2020)Brown, Mann, Ryder, Subbiah, Kaplan, Dhariwal, Neelakantan, Shyam, Sastry, Askell, Agarwal, Herbert-Voss, Krueger, Henighan, Child, Ramesh, Ziegler, Wu, Winter, Hesse, Chen, Sigler, Litwin, Gray, Chess, Clark, Berner, McCandlish, Radford, Sutskever, and Amodei}]{brown2020language}
Tom Brown, Benjamin Mann, Nick Ryder, Melanie Subbiah, Jared~D Kaplan, Prafulla Dhariwal, Arvind Neelakantan, Pranav Shyam, Girish Sastry, Amanda Askell, Sandhini Agarwal, Ariel Herbert-Voss, Gretchen Krueger, Tom Henighan, Rewon Child, Aditya Ramesh, Daniel Ziegler, Jeffrey Wu, Clemens Winter, and 12 others. 2020.
\newblock \href {https://proceedings.neurips.cc/paper_files/paper/2020/file/1457c0d6bfcb4967418bfb8ac142f64a-Paper.pdf} {Language models are few-shot learners}.
\newblock In \emph{Advances in Neural Information Processing Systems}, volume~33, pages 1877--1901. Curran Associates, Inc.

\bibitem[{Bytedance-Seed-Foundation-Code-Team et~al.(2025)Bytedance-Seed-Foundation-Code-Team, :, Cheng, Chen, Chen, Chen, Chen, Chen, Chen, Geng, Li, Li, Li, Li, Liu, Liu, Liu, Liu, Liu, Liu, Liu, Liu, Liu, Long, Mai, Ning, Peng, Shen, Su, Su, Sun, Sun, Tao, Wang, Wang, Wang, Wang, Wang, Xia, Xiang, Xiao, Xiao, Xi, Xin, Xu, Xu, Yang, Yang, Yang, Yuan, Zhang, Zhang, Zhang, Zheng, Zhu, and Zhu}]{cheng2024fullstack}
Bytedance-Seed-Foundation-Code-Team, :, Yao Cheng, Jianfeng Chen, Jie Chen, Li~Chen, Liyu Chen, Wentao Chen, Zhengyu Chen, Shijie Geng, Aoyan Li, Bo~Li, Bowen Li, Linyi Li, Boyi Liu, Jiaheng Liu, Kaibo Liu, Qi~Liu, Shukai Liu, and 37 others. 2025.
\newblock \href {https://arxiv.org/abs/2412.00535} {Fullstack bench: Evaluating llms as full stack coders}.
\newblock \emph{Preprint}, arXiv:2412.00535.

\bibitem[{Cai et~al.(2024)Cai, Jiang, Qin, Cui, Kim, and Huang}]{cai2024shortcut}
Weilin Cai, Juyong Jiang, Le~Qin, Junwei Cui, Sunghun Kim, and Jiayi Huang. 2024.
\newblock Shortcut-connected expert parallelism for accelerating mixture-of-experts.
\newblock \emph{arXiv preprint arXiv:2404.05019}.

\bibitem[{Cassano et~al.()Cassano, Gouwar, Nguyen, Nguyen, Phipps-Costin, Pinckney, Yee, Zi, Anderson, Feldman et~al.}]{cassano2208multipl}
Federico Cassano, John Gouwar, Daniel Nguyen, Sydney Nguyen, Luna Phipps-Costin, Donald Pinckney, Ming-Ho Yee, Yangtian Zi, Carolyn~Jane Anderson, Molly~Q Feldman, et~al.
\newblock Multipl-e: A scalable and extensible approach to benchmarking neural code generation, 2022.
\newblock \emph{URL https://arxiv. org/abs/2208.08227}.

\bibitem[{Chen et~al.(2023)Chen, Liang, Huang, Real, Wang, Pham, Dong, Luong, Hsieh, Lu et~al.}]{chen2023symbolic}
Xiangning Chen, Chen Liang, Da~Huang, Esteban Real, Kaiyuan Wang, Hieu Pham, Xuanyi Dong, Thang Luong, Cho-Jui Hsieh, Yifeng Lu, et~al. 2023.
\newblock Symbolic discovery of optimization algorithms.
\newblock \emph{Advances in neural information processing systems}, 36:49205--49233.

\bibitem[{Cobbe et~al.(2021)Cobbe, Kosaraju, Bavarian, Chen, Jun, Kaiser, Plappert, Tworek, Hilton, Nakano, Hesse, and Schulman}]{cobbe2021training}
Karl Cobbe, Vineet Kosaraju, Mohammad Bavarian, Mark Chen, Heewoo Jun, Lukasz Kaiser, Matthias Plappert, Jerry Tworek, Jacob Hilton, Reiichiro Nakano, Christopher Hesse, and John Schulman. 2021.
\newblock \href {https://arxiv.org/abs/2110.14168} {Training verifiers to solve math word problems}.
\newblock \emph{Preprint}, arXiv:2110.14168.

\bibitem[{DeepSeek-AI et~al.(2024)DeepSeek-AI, Liu, Feng, Wang, Wang, Liu, Zhao, Dengr, Ruan, Dai, Guo, Yang, Chen, Ji, Li, Lin, Luo, Hao, Chen, Li, Zhang, Xu, Yang, Zhang, Ding, Xin, Gao, Li, Qu, Cai, Liang, Guo, Ni, Li, Chen, Yuan, Qiu, Song, Dong, Gao, Guan, Wang, Zhang, Xu, Xia, Zhao, Zhang, Li, Wang, Zhang, Zhang, Tang, Li, Tian, Huang, Wang, Zhang, Zhu, Chen, Du, Chen, Jin, Ge, Pan, Xu, Chen, Li, Lu, Zhou, Chen, Wu, Ye, Ma, Wang, Zhou, Yu, Zhou, Zheng, Wang, Pei, Yuan, Sun, Xiao, Zeng, An, Liu, Liang, Gao, Zhang, Li, Jin, Wang, Bi, Liu, Wang, Shen, Chen, Chen, Nie, Sun, Wang, Liu, Xie, Yu, Song, Zhou, Yang, Lu, Su, Wu, Li, Wei, Zhu, Xu, Huang, Li, Zhao, Sun, Li, Wang, Zheng, Zhang, Xiong, Zhao, He, Tang, Piao, Dong, Tan, Liu, Wang, Guo, Zhu, Wang, Zou, Zha, Ma, Yan, You, Liu, Ren, Ren, Sha, Fu, Huang, Zhang, Xie, Hao, Shao, Wen, Xu, Zhang, Li, Wang, Gu, Li, and Xie}]{liu2024deepseekmla}
DeepSeek-AI, Aixin Liu, Bei Feng, Bin Wang, Bingxuan Wang, Bo~Liu, Chenggang Zhao, Chengqi Dengr, Chong Ruan, Damai Dai, Daya Guo, Dejian Yang, Deli Chen, Dongjie Ji, Erhang Li, Fangyun Lin, Fuli Luo, Guangbo Hao, Guanting Chen, and 138 others. 2024.
\newblock \href {https://arxiv.org/abs/2405.04434} {Deepseek-v2: A strong, economical, and efficient mixture-of-experts language model}.
\newblock \emph{Preprint}, arXiv:2405.04434.

\bibitem[{DeepSeek-AI et~al.(2026)DeepSeek-AI, Xu, Lin, Xue, Wang, Xu, Wu, Zhang, Lin, Dong, Ling, Lu, Zhao, Deng, Hou, Xu, Shao, Ruan, Sun, Dai, Guo, Yang, Chen, Li, Ji, Li, Wei, Lin, Yuan, Xia, Dai, Hao, Chen, Cao, Meng, Li, Yu, Zhang, Xu, Li, Liang, Zhang, Luo, Wei, Yuan, Zhang, Luo, Chen, Ji, Zhang, Ding, Tang, Cao, Gao, Qu, Zeng, Yang, Zhu, Luo, Song, Yu, Huang, Cai, Liang, Zhou, Ye, Li, Xu, Hu, Yang, Chen, Yan, Chen, Zhou, Xiang, Yuan, Cheng, Zhou, Zhu, Yu, Sun, Ran, Jiang, Qiu, Li, Zheng, Song, Dong, Gao, Guan, Zhou, Huang, Yu, Wang, Zhang, Wang, Xia, Zhang, Zhao, Guo, Luo, Ma, Zhu, Wang, Cai, Zhang, Chen, Di, Xu, Mei, Wang, Zhang, Zhang, Tang, Li, Zhou, Han, Wang, Huang, Wang, Cong, Wang, Zhang, Wang, Zhu, Li, Chen, Du, Jiang, Tian, Xu, Lu, Xu, Ge, Zhang, Pan, Wang, Chen, Yin, Xu, Shen, Zhang, Chen, Liu, Lu, Sun, Zhou, Chen, Cai, Nie, Wu, Chen, Hu, Liu, Hu, Ma, Wang, Yu, Zhou, Pan, Yu, Zhou, Ni, Yun, Jin, Pei, Ye, Lin, Ji, Cui, Yue, Yu, Wang, Zhang, Xiao, Zeng, An, Zhao, Liu, Liang, Pang, Luo, Yao,
  Gao, Yang, Huang, Hou, Zhang, Ma, Gao, He, Wang, Wang, Bi, Liu, Wang, Chen, Zhang, Nie, Sun, Wang, Cheng, Liu, Xie, Liu, Liu, Yu, Li, Yang, Zhang, Chen, Wang, Su, Chen, Lin, Fu, Yan, Wang, Ma, Luo, Zhang, Xu, Ma, Huang, Li, Li, Xu, Zhao, Sun, Wang, Qian, Shao, Yu, Zhang, Ding, Shi, Wu, Xiong, Ma, He, Tang, Zhou, Luo, Zhong, Piao, Wang, Zhang, Chen, Tan, Wei, Ma, Liu, Yang, Guo, Wu, Wu, Li, Cheng, Ou, Xu, Li, Wang, Yang, Xu, Wu, Meng, Zou, Zha, Xiong, Chen, Lin, Cao, Wang, Zhang, Yan, Lin, Gu, Luo, You, Liu, Zhou, Zhou, Huang, Wu, Wang, Zhao, Ren, Zhang, Sha, Fu, Ju, Xu, Xie, Zhang, Gao, Hao, Gou, Ma, Yan, Shao, Huang, Chen, Wu, Ren, Wu, Li, Zhang, Xu, Wang, Qu, Gu, Zhu, Li, Zhang, Xie, Gao, Wan, Pan, and Yao}]{deepseekai2026deepseekv4}
DeepSeek-AI, Anyi Xu, Bangcai Lin, Bing Xue, Bingxuan Wang, Bingzheng Xu, Bochao Wu, Bowei Zhang, Chaofan Lin, Chen Dong, Chenchen Ling, Chengda Lu, Chenggang Zhao, Chengqi Deng, Chengyu Hou, Chenhao Xu, Chenze Shao, Chong Ruan, Conner Sun, and 300 others. 2026.
\newblock \href {https://arxiv.org/abs/2606.19348} {Deepseek-v4: Towards highly efficient million-token context intelligence}.
\newblock \emph{Preprint}, arXiv:2606.19348.

\bibitem[{Dua et~al.(2019)Dua, Wang, Dasigi, Stanovsky, Singh, and Gardner}]{dua2019drop}
Dheeru Dua, Yizhong Wang, Pradeep Dasigi, Gabriel Stanovsky, Sameer Singh, and Matt Gardner. 2019.
\newblock Drop: A reading comprehension benchmark requiring discrete reasoning over paragraphs.
\newblock In \emph{Proceedings of the 2019 Conference of the North American Chapter of the Association for Computational Linguistics: Human Language Technologies, Volume 1 (Long and Short Papers)}, pages 2368--2378.

\bibitem[{Gruntkowska et~al.(2025)Gruntkowska, Maziane, Qu, and Richt{\'a}rik}]{gruntkowska2025drop}
Kaja Gruntkowska, Yassine Maziane, Zheng Qu, and Peter Richt{\'a}rik. 2025.
\newblock Drop-muon: Update less, converge faster.
\newblock \emph{arXiv preprint arXiv:2510.02239}.

\bibitem[{Gu et~al.(2024)Gu, Rozi{\`e}re, Leather, Solar-Lezama, Synnaeve, and Wang}]{gu2024cruxeval}
Alex Gu, Baptiste Rozi{\`e}re, Hugh Leather, Armando Solar-Lezama, Gabriel Synnaeve, and Sida~I Wang. 2024.
\newblock Cruxeval: A benchmark for code reasoning, understanding and execution.
\newblock \emph{arXiv preprint arXiv:2401.03065}.

\bibitem[{Gupta et~al.(2018)Gupta, Koren, and Singer}]{gupta2018shampoo}
Vineet Gupta, Tomer Koren, and Yoram Singer. 2018.
\newblock Shampoo: Preconditioned stochastic tensor optimization.
\newblock In \emph{International Conference on Machine Learning}, pages 1842--1850. PMLR.

\bibitem[{He et~al.(2025)He, Han, Zhou, Chen, Liu, Chen, and Wang}]{he2025root}
Wei He, Kai Han, Hang Zhou, Hanting Chen, Zhicheng Liu, Xinghao Chen, and Yunhe Wang. 2025.
\newblock Root: Robust orthogonalized optimizer for neural network training.
\newblock \emph{arXiv preprint arXiv:2511.20626}.

\bibitem[{Hendrycks et~al.(2020)Hendrycks, Burns, Basart, Zou, Mazeika, Song, and Steinhardt}]{hendrycks2020measuring}
Dan Hendrycks, Collin Burns, Steven Basart, Andy Zou, Mantas Mazeika, Dawn Song, and Jacob Steinhardt. 2020.
\newblock Measuring massive multitask language understanding.
\newblock \emph{arXiv preprint arXiv:2009.03300}.

\bibitem[{Hendrycks et~al.(2021)Hendrycks, Burns, Kadavath, Arora, Basart, Tang, Song, and Steinhardt}]{hendrycks2021measuring}
Dan Hendrycks, Collin Burns, Saurav Kadavath, Akul Arora, Steven Basart, Eric Tang, Dawn Song, and Jacob Steinhardt. 2021.
\newblock Measuring mathematical problem solving with the math dataset.
\newblock \emph{arXiv preprint arXiv:2103.03874}.

\bibitem[{Huang et~al.(2023)Huang, Bai, Zhu, Zhang, Zhang, Su, Liu, Lv, Zhang, Fu et~al.}]{huang2023c}
Yuzhen Huang, Yuzhuo Bai, Zhihao Zhu, Junlei Zhang, Jinghan Zhang, Tangjun Su, Junteng Liu, Chuancheng Lv, Yikai Zhang, Yao Fu, et~al. 2023.
\newblock C-eval: A multi-level multi-discipline chinese evaluation suite for foundation models.
\newblock \emph{Advances in neural information processing systems}, 36:62991--63010.

\bibitem[{Jacobs et~al.(1991)Jacobs, Jordan, Nowlan, and Hinton}]{jacobs1991adaptive}
Robert~A Jacobs, Michael~I Jordan, Steven~J Nowlan, and Geoffrey~E Hinton. 1991.
\newblock Adaptive mixtures of local experts.
\newblock \emph{Neural computation}, 3(1):79--87.

\bibitem[{Jain et~al.(2025)Jain, Gu, Li, Yan, Zhang, Wang, Solar-Lezama, Sen, and Stoica}]{jain2025livecodebench}
Naman Jain, Alex Gu, Wen-Ding Li, Fanjia Yan, Tianjun Zhang, Sida Wang, Armando Solar-Lezama, Koushik Sen, and Ion Stoica. 2025.
\newblock Livecodebench: Holistic and contamination free evaluation of large language models for code.
\newblock In \emph{International Conference on Learning Representations}, volume 2025, pages 58791--58831.

\bibitem[{Jordan et~al.(2024)Jordan, Jin, Boza, Jiacheng, Cesista, Newhouse, and Bernstein}]{jordan2024muon}
Keller Jordan, Yuchen Jin, Vlado Boza, You Jiacheng, Franz Cesista, Laker Newhouse, and Jeremy Bernstein. 2024.
\newblock Muon: An optimizer for hidden layers in neural networks, 2024.
\newblock \emph{URL https://kellerjordan. github. io/posts/muon}, 6(3):4.

\bibitem[{Keskar et~al.(2016)Keskar, Mudigere, Nocedal, Smelyanskiy, and Tang}]{keskar2016large}
Nitish~Shirish Keskar, Dheevatsa Mudigere, Jorge Nocedal, Mikhail Smelyanskiy, and Ping Tak~Peter Tang. 2016.
\newblock On large-batch training for deep learning: Generalization gap and sharp minima.
\newblock \emph{arXiv preprint arXiv:1609.04836}.

\bibitem[{Khaled et~al.(2025)Khaled, Ozkara, Yu, Hong, and Park}]{khaled2025muonbp}
Ahmed Khaled, Kaan Ozkara, Tao Yu, Mingyi Hong, and Youngsuk Park. 2025.
\newblock \href {https://doi.org/10.48550/arXiv.2510.16981} {Muonbp: Faster muon via block-periodic orthogonalization}.
\newblock \emph{arXiv preprint arXiv:2510.16981}.

\bibitem[{Kingma and Ba(2014)}]{kingma2014adam}
Diederik~P Kingma and Jimmy Ba. 2014.
\newblock Adam: A method for stochastic optimization.
\newblock \emph{arXiv preprint arXiv:1412.6980}.

\bibitem[{Lai et~al.(2023)Lai, Li, Wang, Zhang, Zhong, Zettlemoyer, Yih, Fried, Wang, and Yu}]{lai2023ds}
Yuhang Lai, Chengxi Li, Yiming Wang, Tianyi Zhang, Ruiqi Zhong, Luke Zettlemoyer, Wen-tau Yih, Daniel Fried, Sida Wang, and Tao Yu. 2023.
\newblock Ds-1000: A natural and reliable benchmark for data science code generation.
\newblock In \emph{International Conference on Machine Learning}, pages 18319--18345. PMLR.

\bibitem[{Li et~al.()Li, Zhang, Koto, Yang, Zhao, Gong, Duan, and Baldwin}]{li2306cmmlu}
Haonan Li, Yixuan Zhang, Fajri Koto, Yifei Yang, Hai Zhao, Yeyun Gong, Nan Duan, and Timothy Baldwin.
\newblock Cmmlu: Measuring massive multitask language understanding in chinese, 2024.
\newblock \emph{URL https://arxiv. org/abs/2306.09212}.

\bibitem[{Li et~al.(2026)Li, Tan, Yang, Sun, Huo, Qin, Zhang, He, Tan, Jia et~al.}]{li2026wisca}
Jiacheng Li, Jianchao Tan, Zhidong Yang, Pingwei Sun, Feiye Huo, Jiayu Qin, Xiangyu Zhang, Maoxin He, Guangming Tan, Weile Jia, et~al. 2026.
\newblock Wisca: A lightweight model transition method to improve llm training via weight scaling.
\newblock In \emph{Findings of the Association for Computational Linguistics: ACL 2026}, pages 1589--1601.

\bibitem[{Li and Hong(2025)}]{li2025note}
Jiaxiang Li and Mingyi Hong. 2025.
\newblock A note on the convergence of muon.
\newblock \emph{arXiv preprint arXiv:2502.02900}.

\bibitem[{Liu et~al.(2024{\natexlab{a}})Liu, Feng, Xue, Wang, Wu, Lu, Zhao, Deng, Zhang, Ruan et~al.}]{liu2024deepseek}
Aixin Liu, Bei Feng, Bing Xue, Bingxuan Wang, Bochao Wu, Chengda Lu, Chenggang Zhao, Chengqi Deng, Chenyu Zhang, Chong Ruan, et~al. 2024{\natexlab{a}}.
\newblock Deepseek-v3 technical report.
\newblock \emph{arXiv preprint arXiv:2412.19437}.

\bibitem[{Liu et~al.(2026)Liu, Zhang, Wang, Hu, Lyu, Sun, Yang, Wang, Li, Qian, Si, Sun, Li, Pei, Xie, and Cai}]{liu2026scalingembeddingsoutperformsscaling}
Hong Liu, Jiaqi Zhang, Chao Wang, Xing Hu, Linkun Lyu, Jiaqi Sun, Xurui Yang, Bo~Wang, Fengcun Li, Yulei Qian, Lingtong Si, Yerui Sun, Rumei Li, Peng Pei, Yuchen Xie, and Xunliang Cai. 2026.
\newblock \href {https://arxiv.org/abs/2601.21204} {Scaling embeddings outperforms scaling experts in language models}.
\newblock \emph{Preprint}, arXiv:2601.21204.

\bibitem[{Liu et~al.(2024{\natexlab{b}})Liu, Xie, Wang, Wei, Ding, and Zhang}]{liu2024evaluating}
Jiawei Liu, Songrun Xie, Junhao Wang, Yuxiang Wei, Yifeng Ding, and Lingming Zhang. 2024{\natexlab{b}}.
\newblock Evaluating language models for efficient code generation.
\newblock \emph{arXiv preprint arXiv:2408.06450}.

\bibitem[{Liu et~al.(2025)Liu, Su, Yao, Jiang, Lai, Du, Qin, Xu, Lu, Yan, Chen, Zheng, Liu, Liu, Yin, He, Zhu, Wang, Wang, Dong, Zhang, Kang, Zhang, Xu, Zhang, Wu, Zhou, and Yang}]{liu2025muon}
Jingyuan Liu, Jianlin Su, Xingcheng Yao, Zhejun Jiang, Guokun Lai, Yulun Du, Yidao Qin, Weixin Xu, Enzhe Lu, Junjie Yan, Yanru Chen, Huabin Zheng, Yibo Liu, Shaowei Liu, Bohong Yin, Weiran He, Han Zhu, Yuzhi Wang, Jianzhou Wang, and 9 others. 2025.
\newblock \href {https://arxiv.org/abs/2502.16982} {Muon is scalable for llm training}.
\newblock \emph{Preprint}, arXiv:2502.16982.

\bibitem[{Loshchilov and Hutter(2017)}]{loshchilov2017decoupled}
Ilya Loshchilov and Frank Hutter. 2017.
\newblock Decoupled weight decay regularization.
\newblock \emph{arXiv preprint arXiv:1711.05101}.

\bibitem[{Martens and Grosse(2015)}]{martens2015optimizing}
James Martens and Roger Grosse. 2015.
\newblock Optimizing neural networks with kronecker-factored approximate curvature.
\newblock In \emph{International conference on machine learning}, pages 2408--2417. PMLR.

\bibitem[{Micikevicius et~al.(2018)Micikevicius, Narang, Alben, Diamos, Elsen, Garcia, Ginsburg, Houston, Kuchaiev, Venkatesh, and Wu}]{micikevicius2017mixed}
Paulius Micikevicius, Sharan Narang, Jonah Alben, Gregory Diamos, Erich Elsen, David Garcia, Boris Ginsburg, Michael Houston, Oleksii Kuchaiev, Ganesh Venkatesh, and Hao Wu. 2018.
\newblock \href {https://arxiv.org/abs/1710.03740} {Mixed precision training}.
\newblock \emph{Preprint}, arXiv:1710.03740.

\bibitem[{Mihaylov et~al.(2018)Mihaylov, Clark, Khot, and Sabharwal}]{OpenBookQA2018}
Todor Mihaylov, Peter Clark, Tushar Khot, and Ashish Sabharwal. 2018.
\newblock Can a suit of armor conduct electricity? a new dataset for open book question answering.
\newblock In \emph{EMNLP}.

\bibitem[{Rein et~al.(2023)Rein, Hou, Stickland, Petty, Pang, Dirani, Michael, and Bowman}]{rein2023gpqa}
David Rein, Betty~Li Hou, Asa~Cooper Stickland, Jackson Petty, Richard~Yuanzhe Pang, Julien Dirani, Julian Michael, and Samuel~R Bowman. 2023.
\newblock Gpqa: A graduate-level google-proof q\&a benchmark.
\newblock \emph{arXiv preprint arXiv:2311.12022}.

\bibitem[{Robbins and Monro(1951)}]{robbins1951stochastic}
Herbert Robbins and Sutton Monro. 1951.
\newblock A stochastic approximation method.
\newblock \emph{The annals of mathematical statistics}, pages 400--407.

\bibitem[{Sakaguchi et~al.(2019)Sakaguchi, Bras, Bhagavatula, and Choi}]{sakaguchi2019winogrande}
Keisuke Sakaguchi, Ronan~Le Bras, Chandra Bhagavatula, and Yejin Choi. 2019.
\newblock Winogrande: An adversarial winograd schema challenge at scale.
\newblock \emph{arXiv preprint arXiv:1907.10641}.

\bibitem[{Sap et~al.(2019)Sap, Rashkin, Chen, LeBras, and Choi}]{sap2019socialIQa}
Maarten Sap, Hannah Rashkin, Derek Chen, Ronan LeBras, and Yejin Choi. 2019.
\newblock \href {https://www.aclweb.org/anthology/D19-1454} {Social iqa: Commonsense reasoning about social interactions}.
\newblock In \emph{EMNLP}.

\bibitem[{Si et~al.(2025)Si, Zhang, and Shen}]{si2025adamuon}
Chongjie Si, Debing Zhang, and Wei Shen. 2025.
\newblock Adamuon: Adaptive muon optimizer.
\newblock \emph{arXiv preprint arXiv:2507.11005}.

\bibitem[{Sutskever et~al.(2013)Sutskever, Martens, Dahl, and Hinton}]{sutskever2013importance}
Ilya Sutskever, James Martens, George Dahl, and Geoffrey Hinton. 2013.
\newblock On the importance of initialization and momentum in deep learning.
\newblock In \emph{International conference on machine learning}, pages 1139--1147. pmlr.

\bibitem[{Suzgun et~al.(2022)Suzgun, Scales, Schärli, Gehrmann, Tay, Chung, Chowdhery, Le, Chi, Zhou, and Wei}]{suzgun2023challenging}
Mirac Suzgun, Nathan Scales, Nathanael Schärli, Sebastian Gehrmann, Yi~Tay, Hyung~Won Chung, Aakanksha Chowdhery, Quoc~V. Le, Ed~H. Chi, Denny Zhou, and Jason Wei. 2022.
\newblock \href {https://arxiv.org/abs/2210.09261} {Challenging big-bench tasks and whether chain-of-thought can solve them}.
\newblock \emph{Preprint}, arXiv:2210.09261.

\bibitem[{Tafjord et~al.(2021)Tafjord, Dalvi, and Clark}]{tafjord2021proofwriter}
Oyvind Tafjord, Bhavana Dalvi, and Peter Clark. 2021.
\newblock Proofwriter: Generating implications, proofs, and abductive statements over natural language.
\newblock In \emph{Findings of the Association for Computational Linguistics: ACL-IJCNLP 2021}, pages 3621--3634.

\bibitem[{Talmor et~al.(2019)Talmor, Herzig, Lourie, and Berant}]{talmor-etal-2019-commonsenseqa}
Alon Talmor, Jonathan Herzig, Nicholas Lourie, and Jonathan Berant. 2019.
\newblock \href {https://doi.org/10.18653/v1/N19-1421} {{C}ommonsense{QA}: A question answering challenge targeting commonsense knowledge}.
\newblock In \emph{Proceedings of the 2019 Conference of the North {A}merican Chapter of the Association for Computational Linguistics: Human Language Technologies, Volume 1 (Long and Short Papers)}, pages 4149--4158, Minneapolis, Minnesota. Association for Computational Linguistics.

\bibitem[{Team et~al.(2026)Team, Bai, Bai, Bao, C., Cai, Cai, Cao, Cao, Chai, Charles, Che, Chen, Chen, Chen, Chen, Chen, Chen, Chen, Chen, Chen, Chen, Chen, Chen, Chen, Chen, Chen, Chen, Chen, Chen, Chen, Chen, Cheng, Cheng, Cui, Cui, Dai, Deng, Ding, Ding, Ding, Dong, Dong, Dong, Dong, Du, Du, Du, Du, Du, Fan, Feng, Feng, Feng, Fu, Fu, Gao, Gao, Gao, Gao, Gao, Geng, Gong, Gong, Gong, Gong, Gu, Gu, Guan, Guo, Guo, Guo, Guo, Hao, Hao, Hao, He, He, He, He, He, He, He, He, He, Hong, Hong, Hu, Hu, Hu, Hu, Hu, Hu, Hua, Huang, Huang, Huang, Huang, Huang, Huang, Huang, Huang, Hui, Jia, Jiang, Jiang, Jiang, Jin, Jin, Jing, Kong, Lai, Li, Li, Li, Li, Li, Li, Li, Li, Li, Li, Li, Li, Li, Li, Li, Li, Li, Li, Li, Li, Li, Li, Li, Li, Lin, Lin, Lin, Lin, Lin, Liu, Liu, Liu, Liu, Liu, Liu, Liu, Liu, Liu, Liu, Liu, Liu, Liu, Liu, Liu, Lu, Lu, Lu, Lu, Lu, Luo, Luo, Luo, Luo, Lyu, Lyu, Mao, Mei, Men, Ni, Niu, Pan, Peng, Qi, Qin, Qin, Qin, Qiu, Qiu, Qiu, Qu, Qu, Shang, Shao, Shen, Shi, Shi, Shi, Shi, Siu, Song, Song, Su, Su,
  Su, Sui, Sun, Sun, Sun, Sun, Sun, Sun, Tai, Tang, Tang, Tang, Tang, Tian, Tian, Tian, Tu, Wang, Wang, Wang, Wang, Wang, Wang, Wang, Wang, Wang, Wang, Wang, Wang, Wang, Wang, Wang, Wang, Wang, Wang, Wang, Wang, Wang, Wang, Wang, Wang, Wang, Wang, Wang, Wang, Wang, Wang, Wang, Wang, Wang, Wang, Wang, Wang, Wang, Wang, Wang, Wang, Wang, Wei, Wei, Wei, Wen, Wu, Wu, Wu, Wu, Wu, Wu, Wu, Wu, Wu, Xian, Xiang, Xiang, Xiao, Xiao, Xiao, Xie, Xie, Xie, Xie, Xing, Xiong, Xu, Xu, Xu, Xu, Xu, Xu, Xu, Xu, Xu, Xu, Xu, Xu, Xu, Xu, Xu, Xu, Xu, Xu, Xue, Yan, Yan, Yang, Yang, Yang, Yang, Yang, Yang, Yang, Yang, Yang, Yang, Yang, Yang, Yang, Yang, Yang, Yang, Yao, Ye, Ye, Ye, Ye, Yin, Yin, Yin, Yu, Yu, Yu, Yu, Yu, Yu, Yuan, Yuan, Yue, Yue, Yue, Zha, Zhan, Zhang, Zhang, Zhang, Zhang, Zhang, Zhang, Zhang, Zhang, Zhang, Zhang, Zhang, Zhang, Zhang, Zhang, Zhang, Zhang, Zhang, Zhang, Zhang, Zhang, Zhang, Zhang, Zhang, Zhang, Zhang, Zhang, Zhang, Zhang, Zhang, Zhang, Zhang, Zhao, Zhao, Zhao, Zhao, Zhao, Zhao, Zhao, Zhao, Zhao, Zhao,
  Zhao, Zheng, Zheng, Zheng, Zheng, Zheng, Zhong, Zhong, Zhong, Zhou, Zhou, Zhou, Zhou, Zhou, Zhou, Zhou, Zhu, Zhu, Zhu, Zhu, Zhu, Zhu, Zhuang, Zhuang, and Zu}]{kimiteam2026kimik3openfrontier}
Kimi Team, Tongtong Bai, Yifan Bai, Yiping Bao, M.~C., Jianfeng Cai, Xinyuan Cai, Peizhou Cao, Yuxuan Cao, Ziwei Chai, Y.~Charles, H.~S. Che, Guanduo Chen, Guangyu Chen, Guanzheng Chen, Huarong Chen, Jia Chen, Jianlong Chen, Jun Chen, and 383 others. 2026.
\newblock \href {https://arxiv.org/abs/2607.24653} {Kimi k3: Open frontier intelligence}.
\newblock \emph{Preprint}, arXiv:2607.24653.

\bibitem[{Team et~al.(2025)Team, Li, Lei, Wang, Rong, Wang, Zhang, Gao, Zhang, Sun et~al.}]{team2025longcat}
Meituan~LongCat Team, Bei Li, Bingye Lei, Bo~Wang, Bolin Rong, Chao Wang, Chao Zhang, Chen Gao, Chen Zhang, Cheng Sun, et~al. 2025.
\newblock Longcat-flash technical report.
\newblock \emph{arXiv preprint arXiv:2509.01322}.

\bibitem[{Vaswani et~al.(2017)Vaswani, Shazeer, Parmar, Uszkoreit, Jones, Gomez, Kaiser, and Polosukhin}]{vaswani2017attention}
Ashish Vaswani, Noam Shazeer, Niki Parmar, Jakob Uszkoreit, Llion Jones, Aidan~N Gomez, {\L}ukasz Kaiser, and Illia Polosukhin. 2017.
\newblock Attention is all you need.
\newblock \emph{Advances in neural information processing systems}, 30.

\bibitem[{Wang et~al.(2024)Wang, Ma, Zhang, Ni, Chandra, Guo, Ren, Arulraj, He, Jiang et~al.}]{wang2024mmlu}
Yubo Wang, Xueguang Ma, Ge~Zhang, Yuansheng Ni, Abhranil Chandra, Shiguang Guo, Weiming Ren, Aaran Arulraj, Xuan He, Ziyan Jiang, et~al. 2024.
\newblock Mmlu-pro: A more robust and challenging multi-task language understanding benchmark.
\newblock \emph{Advances in Neural Information Processing Systems}, 37:95266--95290.

\bibitem[{Xia et~al.(2025)Xia, Shen, Wang, Liu, Sun, Wu, Hu, and Xu}]{xia2025leetcodedataset}
Yunhui Xia, Wei Shen, Yan Wang, Jason~Klein Liu, Huifeng Sun, Siyue Wu, Jian Hu, and Xiaolong Xu. 2025.
\newblock Leetcodedataset: A temporal dataset for robust evaluation and efficient training of code llms.
\newblock \emph{arXiv preprint arXiv:2504.14655}.

\bibitem[{Xie et~al.(2024)Xie, Zhou, Li, Lin, and Yan}]{xie2024adan}
Xingyu Xie, Pan Zhou, Huan Li, Zhouchen Lin, and Shuicheng Yan. 2024.
\newblock Adan: Adaptive nesterov momentum algorithm for faster optimizing deep models.
\newblock \emph{IEEE Transactions on Pattern Analysis and Machine Intelligence}, 46(12):9508--9520.

\bibitem[{Yang et~al.(2025)Yang, Li, Yang, Zhang, Hui, Zheng, Yu, Gao, Huang, Lv, Zheng, Liu, Zhou, Huang, Hu, Ge, Wei, Lin, Tang, Yang, Tu, Zhang, Yang, Yang, Zhou, Zhou, Lin, Dang, Bao, Yang, Yu, Deng, Li, Xue, Li, Zhang, Wang, Zhu, Men, Gao, Liu, Luo, Li, Tang, Yin, Ren, Wang, Zhang, Ren, Fan, Su, Zhang, Zhang, Wan, Liu, Wang, Cui, Zhang, Zhou, and Qiu}]{yang2025qwen3}
An~Yang, Anfeng Li, Baosong Yang, Beichen Zhang, Binyuan Hui, Bo~Zheng, Bowen Yu, Chang Gao, Chengen Huang, Chenxu Lv, Chujie Zheng, Dayiheng Liu, Fan Zhou, Fei Huang, Feng Hu, Hao Ge, Haoran Wei, Huan Lin, Jialong Tang, and 41 others. 2025.
\newblock \href {https://arxiv.org/abs/2505.09388} {Qwen3 technical report}.
\newblock \emph{Preprint}, arXiv:2505.09388.

\bibitem[{You et~al.(2017)You, Gitman, and Ginsburg}]{you2017large}
Yang You, Igor Gitman, and Boris Ginsburg. 2017.
\newblock Large batch training of convolutional networks.
\newblock \emph{arXiv preprint arXiv:1708.03888}.

\bibitem[{You et~al.(2019)You, Li, Reddi, Hseu, Kumar, Bhojanapalli, Song, Demmel, Keutzer, and Hsieh}]{you2019large}
Yang You, Jing Li, Sashank Reddi, Jonathan Hseu, Sanjiv Kumar, Srinadh Bhojanapalli, Xiaodan Song, James Demmel, Kurt Keutzer, and Cho-Jui Hsieh. 2019.
\newblock Large batch optimization for deep learning: Training bert in 76 minutes.
\newblock \emph{arXiv preprint arXiv:1904.00962}.

\bibitem[{Zan et~al.(2022)Zan, Chen, Lin, Guan, Wang, and Lou}]{zan-etal-2022-language}
Daoguang Zan, Bei Chen, Zeqi Lin, Bei Guan, Yongji Wang, and Jian-Guang Lou. 2022.
\newblock When language model meets private library.
\newblock In \emph{Findings of the Association for Computational Linguistics: EMNLP 2022}, pages 277--288.

\bibitem[{Zellers et~al.(2019)Zellers, Holtzman, Bisk, Farhadi, and Choi}]{zellers2019hellaswag}
Rowan Zellers, Ari Holtzman, Yonatan Bisk, Ali Farhadi, and Yejin Choi. 2019.
\newblock Hellaswag: Can a machine really finish your sentence?
\newblock In \emph{Proceedings of the 57th Annual Meeting of the Association for Computational Linguistics}.

\bibitem[{Zhao et~al.(2025)Zhao, Li, Zhou, Tan, and Jia}]{zhao2025alto}
Tong Zhao, Jiacheng Li, Yuanchang Zhou, Guangming Tan, and Weile Jia. 2025.
\newblock \href {https://arxiv.org/abs/2510.27153} {Exploring landscapes for better minima along valleys}.
\newblock \emph{Preprint}, arXiv:2510.27153.

\bibitem[{Zhuo et~al.(2025)Zhuo, Vu, Chim, Hu, Yu, Widyasari, Yusuf, Zhan, He, Paul, Brunner, Gong, Hoang, Zebaze, Hong, Li, Kaddour, Xu, Zhang, Yadav, Jain, Gu, Cheng, Liu, Liu, Wang, Hui, Muennighoff, Lo, Fried, Du, de~Vries, and Werra}]{zhuo2025bigcodebench}
Terry~Yue Zhuo, Minh~Chien Vu, Jenny Chim, Han Hu, Wenhao Yu, Ratnadira Widyasari, Imam Nur~Bani Yusuf, Haolan Zhan, Junda He, Indraneil Paul, Simon Brunner, Chen Gong, Thong Hoang, Armel~Randy Zebaze, Xiaoheng Hong, Wen-Ding Li, Jean Kaddour, Ming Xu, Zhihan Zhang, and 14 others. 2025.
\newblock \href {https://arxiv.org/abs/2406.15877} {Bigcodebench: Benchmarking code generation with diverse function calls and complex instructions}.
\newblock \emph{Preprint}, arXiv:2406.15877.

\end{thebibliography}

\appendix

\section{Proofs of Theoretical Results}\label{sec:proofs}

\subsection{Proof of Lemma~\ref{lem:bounded_accel}}

\begin{proof}
From $A_k = \beta_a A_{k-1} + (1-\beta_a)(G_k - G_{k-1})$ and the triangle inequality,
\begin{equation}
\|A_k\|_F \leq \beta_a\|A_{k-1}\|_F + 2(1-\beta_a)G.
\end{equation}
Unrolling with $A_{-1} = 0$,
\begin{equation}
\|A_k\|_F \leq 2G(1-\beta_a^k) \leq 2G.
\end{equation}
For $\tilde{G}_k = G_k + \alpha A_k$,

\begin{multline}
\|\tilde{G}_k\|_F \leq G + |\alpha| \cdot 2G \\
= G(1 + 2|\alpha|). \qedhere
\end{multline}

\end{proof}

\subsection{Proof of Lemma~\ref{lem:momentum_bound}}

\begin{proof}
The update $M_k = \mu M_{k-1} + \tilde{G}_k$ gives
\begin{equation}
\|M_k\|_F \leq \mu\|M_{k-1}\|_F + G(1 + 2|\alpha|).
\end{equation}
Unrolling with $M_0 = 0$,
\begin{equation}
\|M_k\|_F \leq \frac{G(1 + 2|\alpha|)}{1-\mu}. \qedhere
\end{equation}
\end{proof}

\subsection{Proof of Theorem~\ref{thm:nonconvex}}

\begin{proof}
By L-smoothness,
\begin{multline}
f(W_{k+1}) \leq f(W_k) + \inner{\nabla f(W_k)}{W_{k+1}-W_k} \\
+ \frac{L}{2}\|W_{k+1}-W_k\|_F^2.
\end{multline}
Substituting $W_{k+1} - W_k = -\eta\gamma O_k$,
\begin{multline}
f(W_{k+1}) \leq f(W_k) - \eta\gamma\inner{\nabla f(W_k)}{O_k} \\
+ \frac{\eta^2\gamma^2 L}{2}\|O_k\|_F^2.
\end{multline}

By Assumption~\ref{assum:alignment}, $\EX[\inner{\nabla f(W_k)}{O_k} \mid W_k] \geq \rho\|\nabla f(W_k)\|_F^2$. Thus:
\begin{multline}
\EX[f(W_{k+1}) \mid W_k] \leq f(W_k) - \eta\gamma\rho\|\nabla f(W_k)\|_F^2 \\
+ \frac{\eta^2\gamma^2 L}{2}\EX[\|O_k\|_F^2 \mid W_k].
\end{multline}

For $\|O_k\|_F$: since Newton-Schulz outputs approximately orthogonal matrices with singular values close to~1,

\begin{multline}
\|O_k\|_F^2 \leq C_m \quad \text{for some constant } \\
C_m = O(r).
\end{multline}

Rearranging and taking full expectation,

\begin{multline}
\eta\gamma\rho\EX[\|\nabla f(W_k)\|_F^2] \\
\leq \EX[f(W_k) - f(W_{k+1})] + \frac{\eta^2\gamma^2 L C_m}{2}.
\end{multline}


Summing over $k = 0, \ldots, K-1$ and using $f(W_k) \geq f^*$,

\begin{multline}
\frac{1}{K}\sum_{k=0}^{K-1} \EX[\|\nabla f(W_k)\|_F^2] \\
\leq \frac{f(W_0)-f^*}{\eta\rho\gamma K} + \frac{\eta\gamma L C_m}{2\rho},
\end{multline}
which matches the stated bound with $C_4 = \gamma^2 C_m/2$. With $\eta = O(K^{-1/2})$, this gives $O(K^{-1/2})$ rate.

\end{proof}

\begin{figure}[h!]
\centering
\includegraphics[width=0.85\linewidth]{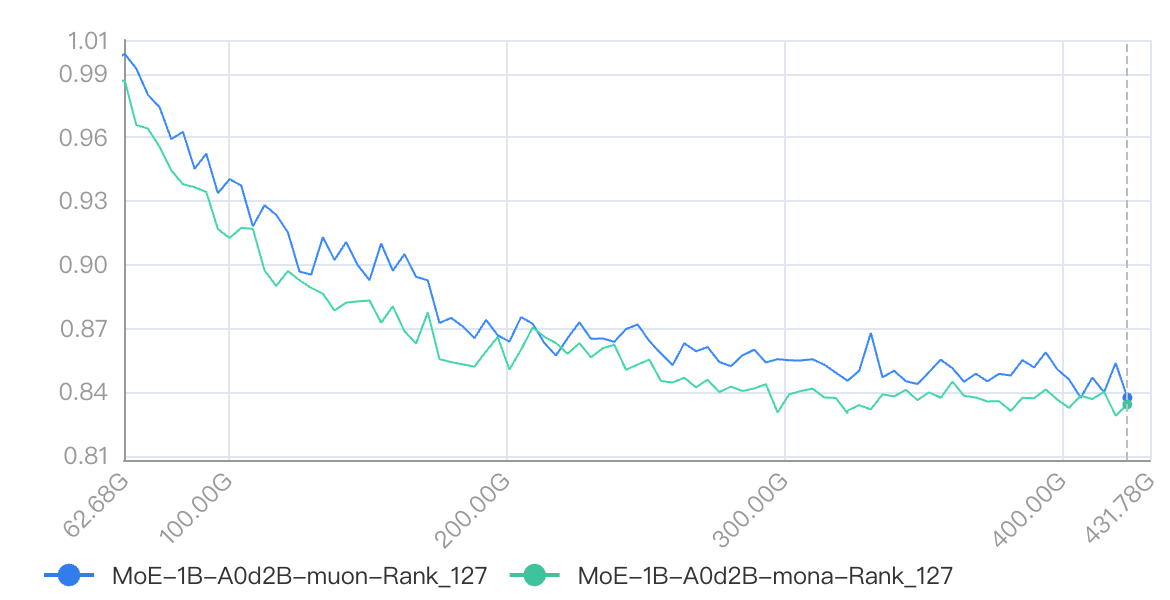}
\caption{Validation loss on Mathematical-Reasoning for MOE-1B-A0d2B.}
\label{fig:200m-gsm8k}
\end{figure}

\begin{figure}[h!]
\centering
\includegraphics[width=0.85\linewidth]{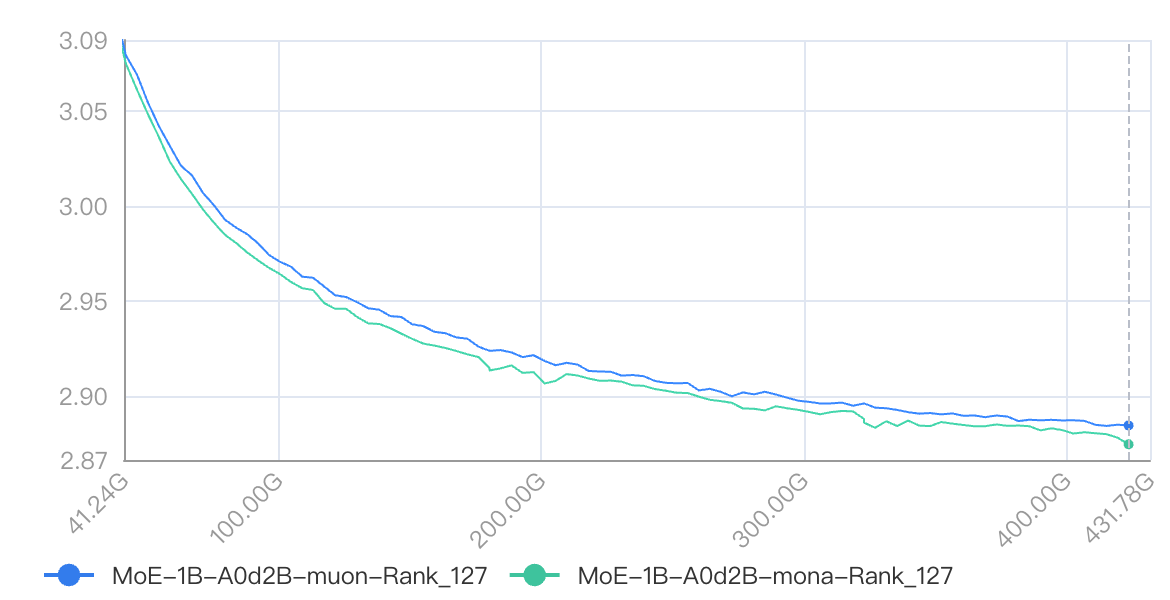}
\caption{Validation loss on Chinese-Academic-Text for MOE-1B-A0d2B.}
\label{fig:200m-qikanwang}
\end{figure}

\begin{figure}[h!]
\centering
\includegraphics[width=0.85\linewidth]{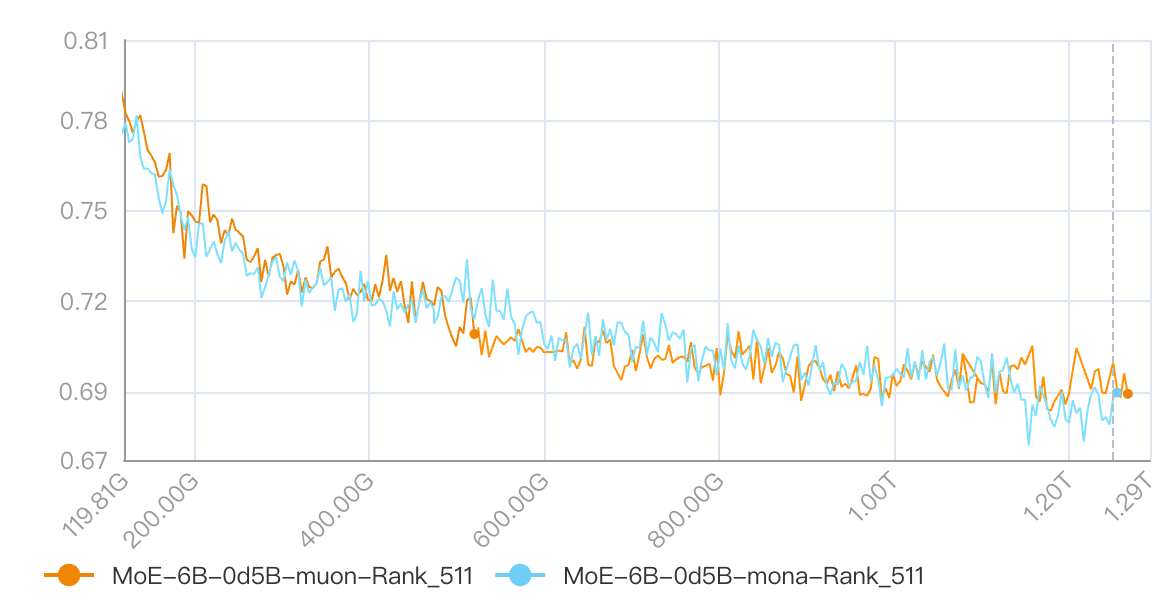}
\caption{Validation loss on Mathematical-Reasoning for MOE-6B-A0d5B.}
\label{fig:500m-gsm8k}
\end{figure}

\begin{figure}[h!]
\centering
\includegraphics[width=0.85\linewidth]{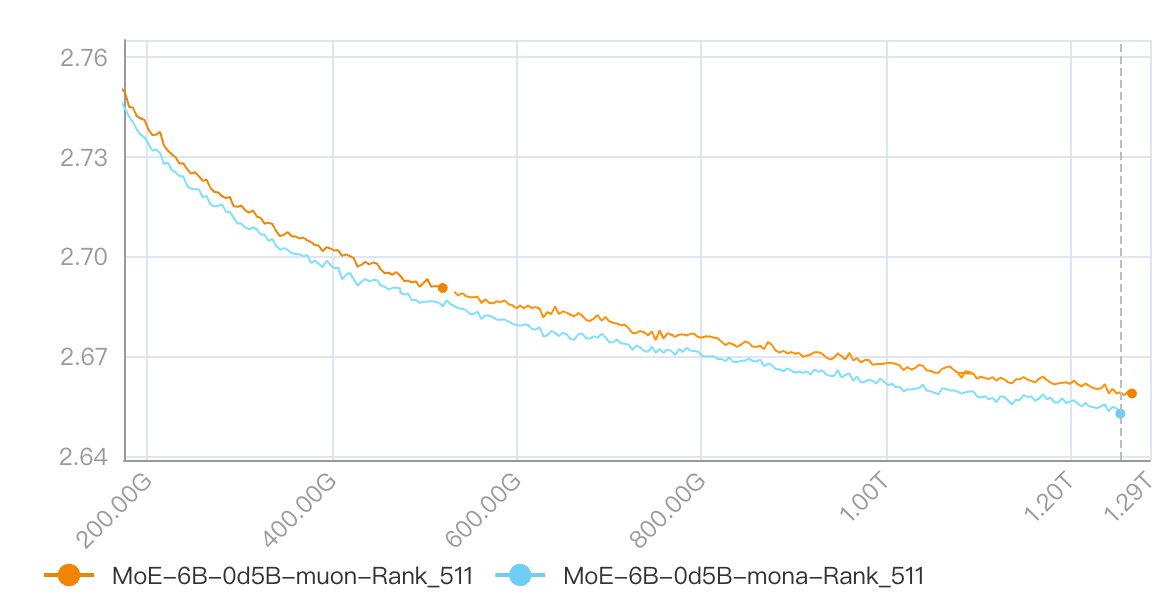}
\caption{Validation loss on Chinese-Academic-Text for MOE-6B-A0d5B.}
\label{fig:500m-qikanwang}
\end{figure}

\subsection{Proof of Proposition~\ref{prop:escape}}

\begin{proof}
Near $W^*$, $\nabla f(W) \approx H^*(W-W^*)$, so

\begin{multline}
\EX[G_k - G_{k-1} \mid W_k, W_{k-1}] \approx H^*(W_k - W_{k-1}) \\
= -\eta\gamma H^* O_{k-1}.
\end{multline}

The acceleration $A_k = \beta_a A_{k-1} + (1-\beta_a)(G_k - G_{k-1})$ yields, taking expectation and unrolling,
\begin{equation}
\EX[A_k] \approx -\eta\gamma H^* \sum_{j=0}^{k}(1-\beta_a)\beta_a^{k-j} O_{j-1}.
\end{equation}

With $H^* = Q\Lambda Q^\top$ and eigenvalues $\lambda_1 \geq \cdots \geq \lambda_d > 0$, in the eigenbasis,
\begin{equation}
\EX[A_k^{(i)}] \approx -\eta\gamma\lambda_i \sum_{j=0}^{k}(1-\beta_a)\beta_a^{k-j} O_{j-1}^{(i)}.
\end{equation}

For large $\lambda_i$ (sharp directions), $|\EX[A_k^{(i)}]|$ is large, pushing the optimizer away. For small $\lambda_i$ (flat directions), it is small, allowing settlement. With $\alpha < 0$, $\tilde{G}_k = G_k + \alpha A_k$ selectively amplifies flat-direction updates. \qedhere
\end{proof}

\section{Additional Validation Curves}\label{sec:additional-curves}

Figures~\ref{fig:200m-gsm8k}--\ref{fig:500m-qikanwang} present supplementary validation loss curves for the mathematical reasoning and Chinese academic text domains.

\section{Computational Overhead Analysis}\label{sec:time-overhead}

The memory overhead of MONA's additional buffers was already discussed in Section~\ref{sec:memory-overhead}. Here, we add measurements of computational time overhead from the MOE-6B-A0d5B pretraining run.

Figure~\ref{fig:0d5b-optimizer-inner-step-time} shows the optimizer inner step time. This is the wall-clock time spent inside the optimizer update at each training step, not counting communication or data loading. MONA adds several simple operations while updating the acceleration item. These add a small amount of time inside the optimizer. By sampling points across the training run, we find that MONA's optimizer inner step time is about \textbf{1\%} higher than Muon's on average.

Figure~\ref{fig:0d5b-iteration-time} reports the end-to-end iteration time, covering the full training step from the forward pass to gradient synchronization and data loading. At this level, the difference between MONA and Muon is too small to notice. The tiny overhead inside the optimizer is completely hidden by the natural variation in communication and data loading between steps. Overall, MONA and Muon run at essentially the same speed in practice.

\begin{figure}[h!]
\centering
\includegraphics[width=0.8\linewidth]{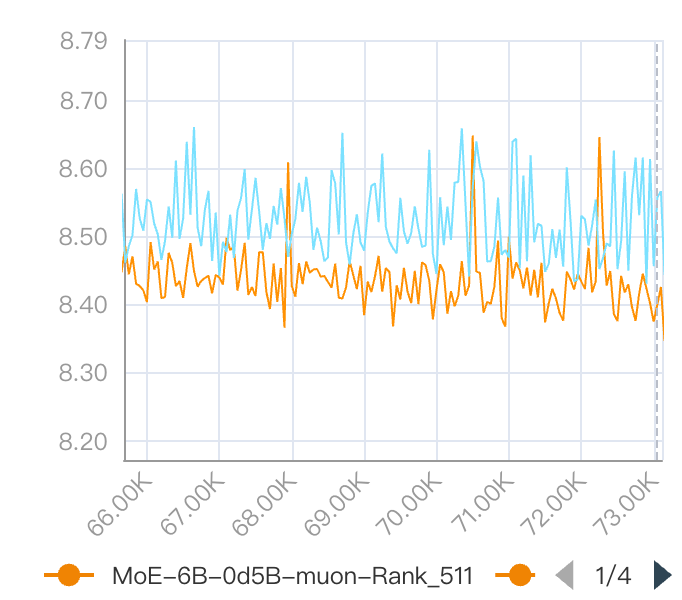}
\caption{Optimizer inner step time for MOE-6B-A0d5B pretraining. MONA runs about 1\% slower than Muon at the optimizer step level.}
\label{fig:0d5b-optimizer-inner-step-time}
\end{figure}

\begin{figure}[h!]
\centering
\includegraphics[width=0.8\linewidth]{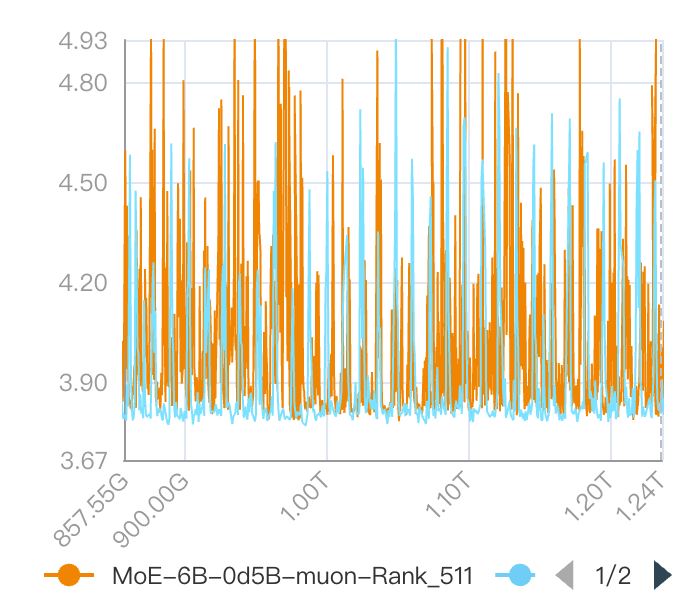}
\caption{End-to-end iteration time for MOE-6B-A0d5B pretraining. MONA and Muon show no practical difference in overall training speed.}
\label{fig:0d5b-iteration-time}
\end{figure}

\section{Comparison with Accelerated AdamW}\label{sec:adamw-acc}

To better understand where the acceleration gains come from, we compare MONA against not only Muon and AdamW but also an AdamW variant equipped with the same acceleration term. We call this variant AdamW-Acc. It is adapted from ALTO's acceleration mechanism, but with one key change for fair comparison in the pretraining setting.

ALTO uses layer-wise learning rate regularization, which assigns different learning rates to different layers. In practice, pretraining typically uses a relatively large fixed learning rate, while post-training applies lower learning rates along with various schedulers. If pretraining already introduces layer-wise learning rate dynamics, it effectively consumes some of the learning rate scheduling benefits that would otherwise belong to the post-training stage. To avoid this, we replace the layer-wise regularization with a default uniform learning rate. This makes AdamW-Acc a clean baseline that isolates the acceleration benefit.

We run all four optimizers, AdamW, AdamW-Acc, Muon, and MONA, on the MOE-1B-A0d2B model under identical settings. Figure~\ref{fig:200m-adamw-muon-2acc-train} shows the training loss curves, and Figure~\ref{fig:200m-adamw-muon-2acc-valid-en_book} shows the validation loss on the General-English-Text domain.

The results follow a consistent ordering. MONA achieves the lowest loss, followed by Muon, then AdamW-Acc, and finally AdamW. This confirms two observations. First, the acceleration term itself brings measurable improvement, since AdamW-Acc outperforms AdamW. Second, Muon already provides a more stable training foundation than AdamW, with smoother loss curves and no spike issues during pretraining. The acceleration term builds on this stable base and pushes performance further, rather than merely recovering from instability.

\section{SAM Sharpness Measurements}\label{sec:sam-sharpness}

To empirically validate that the curvature signal from the acceleration term survives the Newton-Schulz orthogonalization, we measured the SAM (Sharpness-Aware Minimization) sharpness during pretraining of the MOE-6B-A0d5B model. The sharpness metric quantifies how much the loss increases when parameters are perturbed along the gradient direction:
\begin{equation}
\text{Sharpness}(\theta) = f(\theta + \epsilon g / \|g\|) - f(\theta),
\end{equation}
where $f$ is the loss, $g$ is the gradient, and $\epsilon$ is the perturbation radius. We compute sharpness every 1000 training steps on a small training batch (64 samples) with $\epsilon = 0.001$. The perturbation is applied in-place and fully restored after measurement, so it does not affect the training trajectory.

\begin{figure}[h!]
\centering
\includegraphics[width=0.85\linewidth]{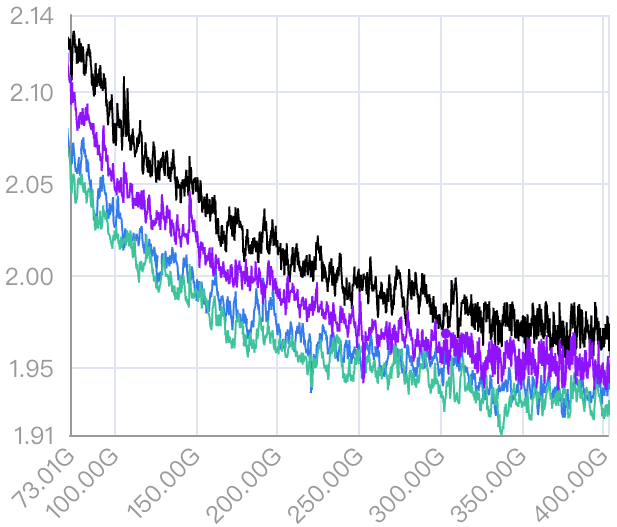}
\caption{Training loss for MOE-1B-A0d2B with four optimizers: AdamW(black), AdamW-Acc(purple), Muon(blue), and MONA(green).}
\label{fig:200m-adamw-muon-2acc-train}
\end{figure}

\begin{figure}[h!]
\centering
\includegraphics[width=0.85\linewidth]{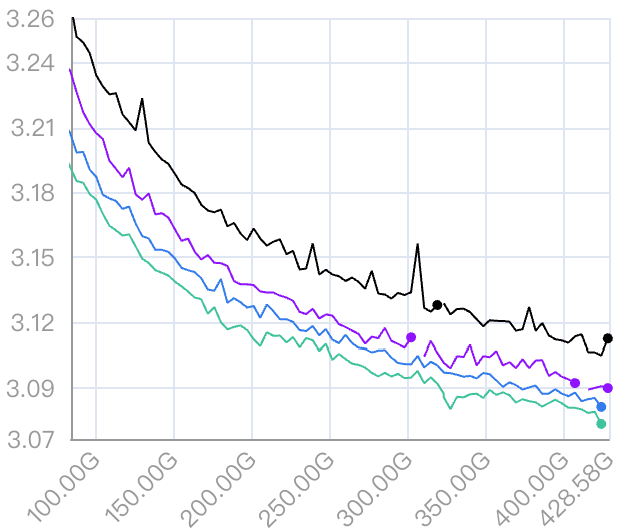}
\caption{Validation loss on General-English-Text for MOE-1B-A0d2B. MONA outperforms Muon, which in turn outperforms AdamW-Acc and AdamW.}
\label{fig:200m-adamw-muon-2acc-valid-en_book}
\end{figure}

Table~\ref{tab:sam-sharpness} reports the results. MONA achieves comparable or slightly lower sharpness than Muon throughout training, confirming that the acceleration term propagates through the orthogonalization step and biases convergence toward flatter regions.

\begin{table}[h!]
\centering
\small
\caption{SAM sharpness comparison during MOE-6B-A0d5B pretraining.}
\label{tab:sam-sharpness}
\begin{tabular}{ccc}
\toprule
\textbf{Step} & \textbf{MONA} & \textbf{Muon} \\
\midrule
0 & 0.000517 & 0.000517 \\
500 & 0.000629 & 0.000719 \\
1000 & 0.000206 & 0.000192 \\
2000 & 0.000153 & 0.000166 \\
3000 & 0.000120 & 0.000121 \\
4000 & 0.000111 & 0.000123 \\
5000 & 0.000093 & 0.000106 \\
6000 ($\sim$100B tokens) & 0.000080 & 0.000083 \\
\bottomrule
\end{tabular}
\end{table}

\end{document}